\documentclass[journal]{IEEEtran}
\usepackage[utf8]{inputenc} 
\usepackage[T1]{fontenc}    
\usepackage{hyperref}       
\usepackage{url}            
\usepackage{booktabs}       
\usepackage{amsmath,amssymb,amsfonts,amsthm,graphicx}       
\usepackage{nicefrac}       
\usepackage{microtype}      
\usepackage{wrapfig}
\usepackage{verbatim}
\usepackage{float}
\usepackage{enumerate}
\usepackage{enumitem}
\usepackage{multirow, array}
\usepackage{authblk}
\usepackage[numbers]{natbib}
\usepackage[symbol]{footmisc}
\usepackage{dsfont}

\renewcommand{\thefootnote}{\fnsymbol{footnote}}

\usepackage{color}

\usepackage{hyperref}
\usepackage{caption}
\usepackage{subcaption}

\usepackage{microtype}

\DeclareGraphicsExtensions{.eps}

\newcommand\nc\newcommand
\nc\bfa{{\boldsymbol a}}\nc\bfA{{\boldsymbol A}}\nc\cA{{\mathcal A}}
\nc\bfb{{\boldsymbol b}}\nc\bfB{{\boldsymbol B}}\nc\cB{{\mathcal B}}
\nc\bfc{{\boldsymbol c}}\nc\bfC{{\boldsymbol C}}\nc\cC{{\mathcal C}}
\nc\sC{{\mathscr C}}
\nc\bfd{{\boldsymbol d}}\nc\bfD{{\boldsymbol D}}\nc\cD{{\mathcal D}}
\nc\bfe{{\boldsymbol e}}\nc\bfE{{\boldsymbol E}}\nc\cE{{\mathcal E}}
\nc\bff{{\boldsymbol f}}\nc\bfF{{\boldsymbol F}}\nc\cF{{\mathcal F}}
\nc\bfg{{\boldsymbol g}}\nc\bfG{{\boldsymbol G}}\nc\cG{{\mathcal G}}
\nc\bfh{{\boldsymbol h}}\nc\bfH{{\boldsymbol H}}\nc\cH{{\mathcal H}}
\nc\bfi{{\boldsymbol i}}\nc\bfI{{\boldsymbol I}}\nc\cI{{\mathcal I}}
\nc\bfj{{\boldsymbol j}}\nc\bfJ{{\boldsymbol J}}\nc\cJ{{\mathcal J}}
\nc\bfk{{\boldsymbol k}}\nc\bfK{{\boldsymbol K}}\nc\cK{{\mathcal K}}
\nc\bfl{{\boldsymbol l}}\nc\bfL{{\boldsymbol L}}\nc\cL{{\mathcal L}}
\nc\bfm{{\boldsymbol m}}\nc\bfM{{\boldsymbol M}}\nc\sM{{\mathscr M}}\nc\cM{{\mathcal M}}
\nc\bfn{{\boldsymbol n}}\nc\bfN{{\boldsymbol N}}\nc\cN{{\mathcal N}}
\nc\bfo{{\boldsymbol o}}\nc\bfO{{\boldsymbol O}}\nc\cO{{\mathcal O}}
\nc\bfp{{\boldsymbol p}}\nc\bfP{{\boldsymbol P}}\nc\cP{{\mathcal P}}
\nc\bfq{{\boldsymbol q}}\nc\bfQ{{\boldsymbol Q}}\nc\cQ{{\mathcal Q}}
\nc\bfr{{\boldsymbol r}}\nc\bfR{{\boldsymbol R}}\nc\cR{{\mathcal R}}
\nc\bfs{{\boldsymbol s}}\nc\bfS{{\boldsymbol S}}\nc\cS{{\mathcal S}}
\nc\bft{{\boldsymbol t}}\nc\bfT{{\boldsymbol T}}\nc\cT{{\mathcal T}}
\nc\bfu{{\boldsymbol u}}\nc\bfU{{\boldsymbol U}}\nc\cU{{\mathcal U}}
\nc\bfv{{\boldsymbol v}}\nc\bfV{{\boldsymbol V}}\nc\cV{{\mathcal V}}
\nc\bfw{{\boldsymbol w}}\nc\bfW{{\boldsymbol W}}\nc\cW{{\mathcal W}}
\nc\bfx{{\boldsymbol x}}\nc\bfX{{\boldsymbol X}}\nc\cX{{\mathcal X}}
\nc\bfy{{\boldsymbol y}}\nc\bfY{{\boldsymbol Y}}\nc\cY{{\mathcal Y}}
\nc\bfz{{\boldsymbol z}}\nc\bfZ{{\boldsymbol Z}}\nc\cZ{{\mathcal Z}}

\nc\diff{{\mathrm d}}
\nc\e{{\mathrm e}}
\nc\calC{{\mathcal C}}

\newcommand{\remove}[1]{}

\newcommand{\avg}{{\mathbb E}}

\newtheorem{conjecture}{Conjecture}
\newtheorem*{lemma*}{Lemma}

\newtheorem{theorem}{Theorem}
\newtheorem{lemma}{Lemma}
\theoremstyle{definition}
\newtheorem*{definition}{Definition}

\newtheorem{remark}{Remark}

\def\DEBUG{true}

\ifdefined\DEBUG

  \def\rem#1{{\marginpar{\raggedright\scriptsize #1}}}
  \newcommand{\barnr}[1]{\rem{\textcolor{red}{$\bullet$ #1}}}
  \newcommand{\aryar}[1]{\rem{\textcolor{green}{$\bullet$ #1}}}
\else

  \newcommand{\barnr}[1]{}
  \newcommand{\aryar}[1]{}

\fi

\allowdisplaybreaks

\usepackage{algorithm}
\usepackage{algorithmic}
\ifCLASSINFOpdf
\else
\fi
\begin{document}
%
\title{Semisupervised Clustering by Queries and Locally Encodable Source Coding}
%
%
%


\author{Arya Mazumdar~\IEEEmembership{Senior Member,~IEEE}, 
\and 
Soumyabrata~Pal}

\maketitle
{\renewcommand{\thefootnote}{}\footnotetext{

 

College of Information and Computer Sciences, University of Massachusetts, Amherst. 
\texttt{\{arya,spal\}@cs.umass.edu}. This work was supported by NSF  Awards CCF-1909046,  CCF-1642658, and CCF-1934846. Some of the results of this paper have appeared in the proceedings of the 2017 Conference on Neural Information Processing Systems.
}
\renewcommand{\thefootnote}{\arabic{footnote}}
\setcounter{footnote}{0}

\begin{abstract}
Source coding is the canonical problem of data compression in information theory. In a {\em locally encodable} source coding, each compressed bit depends on only few bits of the input. In this paper, we show that a recently popular model of semisupervised clustering is equivalent to locally encodable source coding. In this model, the task is to perform multiclass labeling of unlabeled elements. At the beginning, we can ask in parallel a set of simple queries to an oracle who provides (possibly erroneous) binary answers  to the queries. The queries cannot involve more than two (or a fixed constant number of) elements. Now the labeling of all the elements (or clustering) must be performed based on the noisy query answers. The goal is to recover all the correct labelings while minimizing the number of such queries. 
The equivalence to locally encodable source codes leads us to find  lower bounds on the number of queries required in variety of scenarios. We provide querying schemes based on pairwise `same cluster' queries - and  pairwise AND queries, and show provable performance guarantees for each of the schemes.
\end{abstract}


%

\section{Introduction}
\label{sec:Introduction}
Suppose we have $n$ elements, and the $i$th element has a label $X_i \in \{0,1, \dots, k-1\}, \forall i \in \{1,\dots,n\}$. We consider the task of recovering
the labels of the elements (or learning the label vector). This can also be thought of as a clustering problem of $n$ elements into $k$ clusters, where there is a ground-truth
clustering\footnote{The difference between clustering and learning labels is that in the case of clustering it is not necessary to know the value of the label for a cluster. Therefore any unsupervised labeling algorithm will be a clustering algorithm, however the reverse is not true. In this paper we are mostly concerned about the labeling problem, hence our algorithms (upper bounds) are valid for clustering as well.}.
There exist various approaches to this problem in general. In many cases some similarity values between pair of elements are known (a high similarity value indicate that they are in the same cluster). Given these similarity values (or a weighted complete graph), the task is equivalent to graph clustering; when perfect  similarity values are known this is equivalent to finding the connected components of a graph.

A recent approach to clustering, or getting labeled data, has been via crowdsourcing. Suppose there is an oracle (expert labelers, crowd workers) with whom we can make pairwise queries of the form ``do elements $u$ and
$v$ belong to the same cluster?''. We will call this the `same cluster' query (as per \cite{ashtiani2016clustering}).  Based on the answers from the oracle, we then try to reconstruct the labeling or clustering. This idea has seen a recent surge of interest especially in the entity resolution research (see, for e.g. \cite{wang2012crowder,vesdapunt2014crowdsourcing,fss:16,aaai:17}). Since each query to crowd workers cost time and money, a natural objective will be to minimize the number of
queries to the oracle and still recover the clusters exactly. Carefully designed adaptive and  interactive querying algorithms for clustering has also recently been developed \cite{wang2012crowder,vesdapunt2014crowdsourcing,fss:16,mazumdar2017query,mazumdar2017clustering}. In particular, the query complexity for clustering with a $k$-means objective had recently been studied in \cite{ashtiani2016clustering}, and there are significant works in designing optimal crowdsourcing schemes in general (see,  \cite{karger2011iterative,karger2014budget,vempaty2014reliable,zhou2012learning,liu2012variational}). 
Note that, a crowd worker may potentially handle more than two elements at a time; however it is of interest to keep the number of elements involved in a query as small as possible. As an example, recent work in \cite{vinayak2016crowdsourced} considers triangle queries (involving three elements in a query). Also crowd workers can compute some simple functions on this small set of inputs - instead of answering a `same cluster' query. But again it is desirable that the answer the workers provide to be simple, such as a yes/no answer, and also easily computable.

The queries to the oracle can be asked adaptively or non-adaptively.  For the clustering problem, both the adaptive version and the nonadaptive versions have been studied. 
In the above model, $nk$ adaptive `same-cluster' queries are sufficient for exact recovery of all the clusters~\cite{mazumdar2017query}. Whereas, for the non-adaptive case, as will be discussed later, $\Omega(n^2)$ queries are necessary for $k>2$~\cite{mazumdar2017clustering}. This means in absence of any other information, adaptive querying can perform much better than the non-adaptive version. A discussion on how to close this adaptivity gap in query complexity with other assumptions can be found in \cite{mazumdar2017clustering}.
While the adaptive version has this obvious advantage, for crowdsourcing applications it is helpful to have a parallelizable querying scheme in most scenarios for faster response-rate and real time analysis. Indeed, in crowdsourcing, it may take a substantial amount of time to obtain responses from actual crowd-workers and therefore an adaptive algorithm further faces the issue of time as a bottleneck.  In this paper, we concentrate on the nonadaptive version of the problem, i.e., we perform the labeling algorithm after all the query answers are all obtained. 

Budgeted crowdsourcing problems can be quite straight-forwardly viewed as a canonical source-coding or source-channel coding problem of information theory (e.g., see the recent paper \cite{lahouti2016fundamental}). 
A main contribution of our paper is to view this as a {\em locally encodable} source coding problem: a data compression problem where each compressed bit depends only on a constant number of input bits.
The notion of locally encodable source coding is not well-studied even within information theory community, and the only place where it is mentioned to the best of our knowledge is in \cite{montanari2008smooth}, although the focus of that paper is a related notion of {\em smooth} encoding. Another related notion of {\em local decoding} seems to be much more well-studied  \cite{mazumdar2015local,mazumdar2014update,makhdoumi2015locally,patrascu2008succincter,chandar2010sparse,pananjady2015compressing,buhrman2002bitvectors,viola2012bit}.

By posing the querying problem as such we can 
get a number of information theoretic lower bounds on the number of queries required to recover the correct labeling. We also provide 
nonadaptive schemes that are near optimal. Another of our contributions is to show that even within queries with binary answers, `same-cluster' queries (or XOR queries) may not be the best possible choice. A smaller number of queries can be achieved for approximate recovery by using what we call an AND query. Among our settings, we also consider the case when the oracle gives incorrect answers with some probability. 
A simple scheme to reduce errors in this case could be to take a majority vote after asking the same question to multiple different crowd workers. However, often that is not sufficient. Experiments on several real datasets (see \cite{mazumdar2017clustering}) with answers collected from Amazon Mechanical Turk \cite{DBLP:journals/corr/GruenheidNKGK15,DBLP:conf/icde/VerroiosG15} show that majority voting could  even increase the errors.
Interestingly, such an observation has been made by a recent paper as well \cite[Figure 1]{prelec2017solution}. This prompts different, and more involved aggregating schemes than majority voting that are difficult to theoretically analyze. The probability of error of a query answer may also be thought of as the aggregated answer after repeating the query several times. Once the answer has been aggregated, it cannot change -- and thus it reduces to the model where repeating the same question multiple times is not allowed. On the other hand, it is usually assumed that the answers to different queries are independently erroneous (see \cite{gruenheid2015fault}). Therefore we consider the case where repetition of a same query multiple times is not allowed\footnote{Independent repetition of queries is also theoretically not  interesting, as by repeating any query just $O(\log n)$ times one can reduce the probability of error to near zero.}, however different queries can result in erroneous answers independently. 


In this case, the best known polynomial-time algorithms need $O(nk^2 \log n)$ queries to perform the clustering with $k$ clusters correctly with high probability~\cite{mazumdar2017clustering}. We extend this result to show that if we are allowed to tolerate $\delta$ proportion of elements being wrongly clustered, then there exists a polynomial time querying scheme with $O(nk^2 \log\frac{k}{\delta})$ queries.  We also show that by employing a generalized version of AND querying method, $(1-\delta)$-proportion of all labels in the label vector can be efficiently recovered with only $O(nk \log \frac{k}\delta)$ queries. In this generalized AND query we ask queries of the form: `what is the common label of these elements?'.

Along the way, we also provide new information theoretic results on fundamental limits of locally encodable source coding. While the the related notion of locally decodable source code \cite{mazumdar2015local,makhdoumi2015locally,patrascu2008succincter,chandar2010sparse}, as well as smooth compression \cite{montanari2008smooth,patrascu2008succincter} have been studied, there was no nontrivial result known related to locally encodable codes in general.    
Although the focus of this paper is primarily theoretical, we also perform a small but real crowdsourcing experiment to validate our algorithm. 

A preliminary version of this paper presented in a conference~\cite{mazumdar2017semisupervised} contains an error. Namely Theorem 1 therein was wrong, as pointed out by~\cite{pang2019coding}. The correct statement in  the relevant regime is provided in Section~\ref{subsec:Perfect+Full}.
\section{Problem Setup and Information Theoretic View}
\label{sec:Setup}
For $n$ elements, consider a label vector $\bfX \in \{0, \dots, k-1\}^n$, where $X_i$, the $i$th entry of $\bfX$, is the label of the $i$th element and can take one of $k$ possible values. Suppose $P(X_i =j) = p_j \forall j$ and $X_i$'s are independent.  In other words, the prior distribution of the labels is given by the vector $\bfp \equiv (p_0, \dots, p_{k-1})$. 
  For the special case of $k=2$, we denote $p_0 \equiv 1-p$ and $p_1 \equiv p$.
While we emphasize on the case of $k=2$ our results extends in the case of larger $k$, as will be mentioned. 

A query $Q: \{0, \dots, k-1\}^\Delta \to \{0,1\}$ is a deterministic function that takes as argument $\Delta$ labels, $\Delta\ll n$, and 
outputs a binary answer. While the query answer need not be binary, we restrict ourselves mostly to this case for being a practical choice. 

Suppose a total of $m$ parallel i.e., {\em nonadaptive}, queries are made and the query answers are given by $\bfY \in \{0,1\}^m$. The objective is to reconstruct the label vector $\bfX$ from $\bfY$, such that the number of queries $m$ is minimized.

We assume our recovery algorithms to have the knowledge of $\bfp$. This prior distribution, or the relative sizes of clusters, is usually easy to estimate by subsampling a small ($O(\log n)$) subset of elements and performing a complete clustering within that set (by, say, all pairwise `same-cluster' queries). In many prior works, especially in the recovery algorithms of popular statistical models such as stochastic block model,  it is assumed that the relative sizes of the clusters are known (see \cite{abh:16}).

 
We also consider the setting where query answers may be erroneous with  some probability of error. For crowdsourcing applications, this is a valid assumption since many times even expert labelers can make errors, and such assumption can be made. To model this we assume each entry of $\bfY$ is flipped independently with some probability $q$. Such independence assumption has been used many times previously to model 
errors in crowdsourcing systems (see, e.g., \cite{gruenheid2015fault}). While this may not be the perfect model, we {\em do not allow a single query to be repeated multiple times in our algorithms} (see the Introduction for a justification).  
For the analysis of our algorithm we just need to assume that the answers to different queries are independent.
While we analyze our algorithms under these assumptions for theoretical guarantees, it turns out that even in real crowdsourcing situations our algorithmic results mostly follow the theoretical  
results, giving further validation to the model.

For the $k=2$ case, and when $q=0$ (perfect oracle), it is easy to see that $n$ nonadaptive `same-cluster' queries are sufficient for the task. One simply compares every element with the  first element. This does not extend to the case when $k>2$: for perfect recovery, and without any prior, one must make  $\Omega(n^2)$ nonadaptive pairwise   queries in this case (Claim 4 in~\cite{mazumdar2017clustering}).  When $q>0$ (erroneous oracle), it has been shown that a total number of  $O(\gamma nk \log n)$ nonadaptive queries are sufficient \cite{mazumdar2017clustering}, where $\gamma$ is the ratio of the sizes of the largest and smallest clusters, albeit with an inefficient algorithm.  

{\bf Note that, all the statements regarding number of queries in this paper refers to nonadaptive schemes.} So we will omit that qualifier going forward.


\paragraph{Information theoretic view.} 
The problem of learning a label vector $\bfX$ from queries is very similar to the canonical source coding (data compression) problem from information theory. 
In the source coding problem, a (possibly random) vector $\bfX$ is `encoded' into a usually smaller length binary vector called the {\em compressed vector}\footnote{The compressed vector is not necessarily binary, nor it is necessarily smaller length.} $\bfY \in \{0,1\}^m$. The decoding task is to again obtain $\bfX$ from the compressed vector $\bfY$. It is known that if each entry of $\bfX$ is independently distributed according to $\bfp$, then $m \approx n H(\bfp)$ is both necessary and sufficient to recover $\bfx$ with high probability, where $H(\bfp) = - \sum_ip_i \log p_i$ is the entropy of probability vector $\bfp$. 

We can cast our problem in this setting naturally, where entries of $\bfY$ are just answers to queries made on $\bfX$. The main difference is that in source coding each $Y_i$ may potentially depend on all the entries of $\bfX$; while in the case of label learning, each $Y_i$ may depend on only $\Delta$ of the $X_i$s.  

We call this {\em locally encodable source coding}. This terminology is analogous to the recently developed literature on locally decodable source coding \cite{mazumdar2015local,makhdoumi2015locally}. It is called locally encodable, because each compressed bit depend only on $\Delta$ of the source (input) bits. For locally decodable source coding, each bit of the reconstructed sequence $\hat{\bfX}$ depends on at most a prescribed constant number $\Delta$ of bits from the compressed sequence. Another closely related notion is that of `smooth compression', where each source bit contributes to at most $\Delta$ compressed bits \cite{montanari2008smooth}. Indeed, in \cite{montanari2008smooth}, the notion of  locally encodable source coding is also present where it was called robust compression. 
We provide new information theoretic  lower bounds on the number of queries required to reconstruct $\bfX$ exactly and approximately for our problem.

For the case when there are only two labels, the `same-cluster' query is equivalent to a Boolean XOR operation between the labels. There are some inherent limitations to these functions that prohibit the `same-cluster' queries to achieve the best possible number of queries for the `approximate' recovery of labeling problem. We use an old result by Massey \cite{massey1977joint} to establish this limitation.  We show that, instead using an operation like Boolean AND, much smaller number of queries are able to recover most of the labels correctly.


We also consider the case when the oracle gives faulty answer, or $\bfY$ is corrupted by some noise (the {\em binary symmetric channel}). This setting is analogous to the problem of  {\em joint source-channel coding}. However, just like before, each encoded bit must depend on at most $\Delta$ bits.  In a real crowdsourcing experiment, we have seen that if crowd-workers have been provided with the same set of pairs and being asked for `same cluster' queries as well as AND queries, the error-rate of AND queries is lower, when there are two groundtruth clusters. The reason is that for a correct `no' answer in an AND query, a worker just needs to know one of the labels in the pair. For a `same cluster' query, both the labels must be known to the worker for any correct answer. We show that for the approximate recovery problem, AND queries are again performing substantially well, and provide theoretical guarantees for both AND query and `same cluster' query based schemes when two or more groundtruth clusters are present.

There are multiple reasons why one would ask for a `combination' or function of multiple labels from a worker instead of just asking for a label itself (a `label-query'). Information theoretically, a `label-query' requires more information (up to $\log k$ bits) from the workers compared to `same-cluster' queries (at most 1 bit). Further, a crowd worker can answer to `same-cluster' queries without being aware of the complete label set. 
While these two advantages are not present in the generalized AND queries we describe, 
 in case of erroneous answer with AND queries (or `same cluster' queries), we have the option of not repeating a query, and still reduce errors. No such option is available with direct label-queries. 
 Furthermore, as we will  subsequently see,  one needs less number of AND queries than both direct label queries and `same-cluster' queries, for approximate recovery in certain regime.

\begin{remark}\label{rem:cl}
Note that, using only `same-cluster' queries  at best the complete clustering can be recovered, and it is not possible to recover the labeling unless some other information is also available. Indeed, if the prior $\bfp = (p_0, \dots, p_{k-1})$ is known, and $p_i \ne p_j \forall i,j$, then by looking at the complete clustering it is possible to figure out the labeling (or to correctly assign the clusters their respective labels, since the label vector must belong to the `typical set' of $\bfp$. This implies, for the case of two clusters, if $p \ne \frac12$ then the labeling can be inferred with high probability from the clustering.
\end{remark}

\paragraph{Contributions.}
In summary our contributions can be listed as follows.

\noindent 1. Noiseless queries and exact recovery (Sec.~\ref{subsec:Perfect+Full}): For two labels/clusters, we provide a querying scheme that asks $n - \Theta(n/ \log n)$ nonadaptive pairwise `same cluster' queries, and recovers the  labels with high probability, for a range of prior probabilities. We also show that, this result is order-wise optimal.  If instead we involve $\Delta \ge 3$ elements in each of the queries, then with 
$\alpha n, \alpha<1$ number of nonadaptive XOR queries we recover all labels with high probability, for a range of prior probabilities. We also provide a new lower bound that is strictly better than $n H(\bfp)$ for some $\bfp$.

\noindent 2. Noiseless queries and approximate recovery (Sec.~\ref{subsec:Perfect+Approximate}): We provide a new lower bound on the number of queries required to recover $(1-\delta)$ fraction of the labels $\delta >0$. We also show that `same cluster' queries are insufficient for certain regime of $\delta$, and propose a new querying strategy based on AND operation that performs substantially better.

\noindent  3. Noisy queries and approximate recovery (Sec.~\ref{subsec:Faulty+Approximate}). For this part we assumed the query answer to be $k$-ary ($k \ge 2$) where $k$ is the number of clusters. This section contains two main algorithms that use the `same-cluster' queries and AND queries as main primitives repectively. We show that, even in the presence of noise in the query answers, it is possible to recover $(1-\delta)$ proportion of all labels correctly with only $O(nk^2 \log \frac{k}{\delta})$ nonadaptive `same-cluster' queries and $O(nk \log \frac{k}{\delta})$ AND queries. We validate this theoretical result in a  crowdsourcing experiment in Sec.~\ref{sec:Experiments}.

\section{Main results and techniques}
\subsection{Noiseless queries and exact recovery}
\label{subsec:Perfect+Full}
In this scenario we assume the query answer from the oracle to be perfect and we wish to get back all of the original labels exactly without any error. Each query is allowed to take only $\Delta$ labels as input. When $\Delta=2$, we are allowed to ask only pairwise queries. Let us consider the case when there are only two labels, i.e., $k=2$.
 That means the labels  $X_i \in \{0,1\}, 1\le i\le n,$ are iid Bernoulli($p$) random variable.  Therefore the number of queries $m$ that are necessary and sufficient to recover all the labels with high probability  is approximately $n H(p) \pm o(n)$ where for a scalar $x$, we define $H(x)\equiv -x \log_2 x -(1-x) \log_2 (1-x)$ to be the binary entropy function. However the sufficiency part here does not take into account that each query can involve only $\Delta$ labels. 
\subsubsection{Same cluster queries ($\Delta=2$)}
We warm up with the case of only two labels (clusters), and the same cluster queries.
For two labels, a same cluster query simply amounts to being the modulo 2 sum of the two label values. It is easy to see that querying the first label ($X_1$) with every other label allows us to infer the clustering of the labels since we can simply group the labels which are same as the first label and the labels which are different to the first label separately. 
As mentioned in Remark \ref{rem:cl}, if $p \ne \frac12$, then just by counting the sizes of the groups, it is possible to get the labels correctly as well (among the two possible labelings, given the clustering).
The query complexity in this case is $n-1$ but we can have the following scheme that uses less than $n-1$ queries by building on the aforementioned idea. 

\paragraph{Querying scheme.} Our scheme works in the following manner. First, we partition the $n$ elements into $d$ disjoint and equal sized groups each containing $\frac{n}{d}$ elements (assume $d | n$). For each group, we query the first element of the group with every other element of the group. Now, for each group we can cluster the labels and we will identify
\begin{itemize}
\item the smaller cluster with the label \texttt{1} and the larger cluster with the label \texttt{0}, when $p <\frac12$
\item the smaller cluster with the label \texttt{0} and the larger cluster with the label \texttt{1}, when $p > \frac12$.
\end{itemize}
Going forward, let us just assume $p <\frac12$ without any loss of generality.
Finally we can aggregate all the elements identified with the label \texttt{1}  and with label \texttt{0} and return them as our final output. The total number of queries required in this scheme is $d(\frac{n}{d}-1) = n-d$.

\begin{theorem}\label{thm:-1}
For the querying scheme described above, $n-\frac{nD(\frac{1}{2}||p)}{2\log n}$ `same cluster' queries are sufficient to recover the label vector $\bfX$ with probability at least $1-\frac{D(\frac{1}{2}||p)}{2n\log n}$ where $D(p||q)\equiv p\ln \frac{p}{q}+(1-p)\ln \frac{1-p}{1-q}$ is the Kullback-Leibler Divergence between two Bernoulli distributed random variables with parameters $p$ and $q$ respectively.
\end{theorem}
\begin{proof}
Let us set $d =\frac{nD(\frac{1}{2}||p)}{2\log n}.$ We omit the use of ceilings and floor to maintain clarity.  Within a group of elements, we are able to obtain the complete clustering. We will fail to obtain the labeling only if the larger cluster has the true label \texttt{1}. However this happens with probability at most $2^{-\frac{nD(\frac{1}{2}||p)}{d}}$ \cite{cover2012elements}. Since there are $d$ groups, the probability that we fail to recover the labels in any one of the groups is at most $d2^{-\frac{nD(\frac{1}{2}||p)}{d}} = \frac{d}{n^2} = \frac{D(\frac{1}{2}||p)}{2n\log n}$. The total number of queries is $n -d = n-\frac{nD(\frac{1}{2}||p)}{2\log n}$.
%
\end{proof}
Now, we will show a matching lower bound that proves that the reduction on query complexity presented in the scheme described above to be tight up to constant factors. 
In particular, we prove the following theorem.
\begin{theorem}\label{thm:d2}
Consider the binary labeling problem with pairwise queries. If the number of `same cluster' queries is less than $n-\frac{(2+\epsilon)nD(\frac{1}{2}||p)}{\log n}$ for any positive constant $\epsilon>0$, then 
any querying scheme will fail to recover the labels of all elements with positive probability.
\end{theorem}
To prove this theorem we will need the following lemma.
\begin{lemma}[Theorem 11.1.4 in \cite{cover2012elements}]\label{lem:lbc}
Consider a vector $\mathbf{X} \in \{0,1\}^{n}$ whose  elements are i.i.d random variables sampled  according to $\textup{Ber}(p)$, $p <\frac12$. The probability that more than $\frac{n}{2}$  are \texttt{1} is  at least  $\frac{2^{-nD(\frac{1}{2}||p)}}{(n+1)^{2}}$.   
\end{lemma}

\begin{proof}[Proof of Theorem~\ref{thm:d2}]
Suppose a querying scheme uses $n-a$ pairwise `same cluster' queries. Consider a graph with $n$ vertices corresponding to the querying scheme. The vertices are labeled $1, \dots, n$, and the edge $(i,j)$ exists if and only if $(i,j)$ is a query. Since the number of edges in the graph is $n-a$,  the graph has at least $a$ components. Therefore average size of any component in the graph is at most $\frac{n}{a}$. This also implies that there exists at least $\frac{a}{2}$ components with size at most $\frac{2n}{a}$ each (using Markov inequality).

Consider the elements corresponding to vertices of one such component. Even if all the possible `same cluster' queries were made within this group, we will still have two possible labelings that would be consistent with all possible query answers (for every assignment of labels, a new assignment can be created by flipping all the labels). Since this set of elements are not queried with any element outside of it, this will give rise to $2^{a/2}$ different possibilities. This situation can only be mitigated if we can turn the clustering into labeling. Since there are two possible labelings within a group, we will make a mistake in labeling only when the number of elements with label \texttt{1} is larger than the number of elements with label \texttt{0} (the maximum likelihood decoding will fail).

Now, using Lemma~\ref{lem:lbc}, within a component of size at most $\frac{2n}{a}$, the probability that the number of elements with label \texttt{1} is less than the number of elements with label \texttt{0}  is at most 
$$
1-\frac{2^{-2nD(\frac{1}{2}||p)/a}}{(2n/a+1)^{2}}.
$$
And the probability that this happens for all $\frac{a}{2}$ such components is at most

\begin{align*}
&\Big(1-\frac{2^{-2nD(\frac{1}{2}||p)/a}}{(2n/a+1)^{2}}\Big)^{a/2} \\
&\le \exp\Big(-\frac{a2^{-2nD(\frac{1}{2}||p)/a}}{2(2n/a+1)^{2}}\Big)  = o(1),
\end{align*}

if we substitute $a= \frac{(2+\epsilon)nD(\frac{1}{2}||p)}{\log n}$ for any positive constant $\epsilon>0$. 

\end{proof}
The main takeaway from this part is that, by exploiting the prior probabilities (or relative cluster sizes), it is possible to infer the labels with strictly less than $n$ `same cluster' queries. However, to make the deduction  with $o(n)$ queries we need to look at either a different type of querying, or involve more than two elements in a query.

\subsubsection{XOR queries for larger $\Delta$}

\noindent{\bf Querying scheme:} We use the following type of queries. For each query, labels of  $\Delta$ elements are given  to the oracle,  and the oracle returns a simple XOR operation of the labels. Note, for $\Delta =2$, our queries are just `same cluster' queries. Let us define $\mathbf{Q}$ to be the binary query matrix of dimension $m \times n$ where each row has at most $\Delta$ ones, other entries being zero. Now for a label vector $\mathbf{X}$ we can represent the set of query outputs by $\mathbf{Y} = \mathbf{Q}\mathbf{X} \pmod{2}$. In order to fulfill our objective, we will define a random ensemble of query matrices $\mathcal{Q}$ from which $\mathbf{Q}$ is sampled and subsequently, we will show that the average probability of error goes to zero asymptotically with $n$ that implies the existence of a good matrix $\mathbf{Q}$ in the ensemble. 

\noindent \emph{Random ensemble: } The random ensemble $\mathcal{Q}$ will be defined in terms of a bipartite graph. This is done by constructing a biregular  bipartite graph $G(V_1 \sqcup V_2, E)$, with $|V_1| =n, |V_2| =m$. Here $V_1,$ called the left vertices, represents the labels; and $V_2,$ the right vertices, represents the query. The degree of each left vertex is $c$ and the degree of each right vertex is $\Delta$. We  have $|E| = m\Delta=nc$. Now, a permutation $\pi$ is randomly sampled from the set of all permutation on $\{1,2,\dots,nc\}$, the symmetric group $S_{nc}$. If we fix an ordering of the edges emanating from vertices in the left and right, then $i$th edge will be joined with the $\pi(i)$th edge on the right. 

\noindent \emph{Decoder:} In this setting we will be concerned with the exact recovery of the labels. The decoder $\Psi:\{0,1\}^{m} \to \{0,1\}^{n}$ will look at the vector $\mathbf{Q}\mathbf{X}$ and return a vector $\hat{\mathbf{X}}$ such that the Hamming weight (number of nonzero entries) of the vector is given by $|\textup{wt}(\hat{\mathbf{X}}) -np| \le n^{2/3}$ and $\mathbf{Q}\hat{\mathbf{X}} = \mathbf{Q}\mathbf{X}$. Hence, the probability of error $P_e$ can be defined as 
\begin{align*}
P_e\equiv \sum_{\mathbf{X} \in \{0,1\}^{n}} \Pr(\mathbf{X})\Pr_{\mathbf{Q} \sim \mathcal{Q}}(\Psi(\mathbf{Q}\mathbf{X}) \neq \mathbf{X}).
\end{align*}

We have the following theorem
\begin{theorem}\label{thm:1}
Consider an $m \times n$ query matrix $\mathbf{Q}$ sampled from the  ensemble $\mathcal{Q}$ of   matrices  described above, and a label vector $\bfX$ with each entry being i.i.d. $\textup{Ber}(p), p < \frac12$. Let $c,\Delta$ be the left  and the right degrees of the ensemble with $3 \le c < \Delta$ and let $\beta=\frac{2}{\Delta}\Big(\frac{1}{2\Delta^2p(1-p)e^{3c/2}}\Big)^{\frac{1}{c-2}}$. If the number of queries
\begin{align*}
m > n\max_{\beta \le x \le 2p} \frac{pH\Big(\frac{x}{2p}\Big)+(1-p)H\Big(\frac{x}{2(1-p)}\Big)}{1-\log \Big(1+\Big(1-2x \Big)^{\Delta}\Big)},
\end{align*}
then the average probability of error $P_e$ goes to zero as
\begin{align*}
P_e \le n^{2-c}(1-p+p^{2})(\Delta c)^{c}(1+o(1)).
\end{align*}
\end{theorem} 
The proof of this theorem is delegated to Section~\ref{sec:thm1}. The same ensemble $\mathcal{Q}$ was used by \cite{miller2001bounds} where the authors showed the existence of linear codes achieving zero probability of error in the Binary Symmetric Channel such that the parity check matrix of the code belonged to the ensemble $\mathcal{Q}$. Because of the duality of source-channel coding, their guarantees on the average probability of error for $\mathcal{Q}$ directly translate to our setting as well. However, our analysis is slightly different from \cite{miller2001bounds} and it is tighter in most cases. We have delegated the detailed comparison of the two analyses to Section~\ref{sec:thm1}. 
The achievability result is depicted in Figure~\ref{fig:first}.

\subsubsection{Lower bounds (converse)}
\label{sec:lower_bounds}
Now we provide some necessary conditions on the number of queries, involving $\Delta$ elements at most, required for a full recovery of labels. First of all notice that, if a query involves at most $\Delta$ elements, then  $\lceil \frac{n}{\Delta} \rceil$ queries are necessary for exact recovery. If a particular label is not present in any query, then the decoder has no other choice but to guess the label. This will lead to a constant probability of error $\min(p,1-p)$. Therefore at least $\lceil{ \frac{n}{\Delta}} \rceil$ queries are necessary so that every label is present in at least one query. 

Adapting Gallager's result for low density parity-check matrix codes for our setting (using source-channel duality for linear codes) we can have the following lower bound on the number of queries.

\begin{theorem}[\cite{gallager1962low}]\label{thm:Gallager}
Assume a  label vector $\bfX$ with each entry being i.i.d. $\textup{Ber}(p), p < \frac12$. If the number of XOR queries, each involving at most $\Delta$ labels, is  less than $\frac{n H(p)}{H(\frac{1+(1-2p)^{\Delta}}{2})}$, then the probability of error in recovery of labels is bounded from below by a constant independent of the number of elements $n$.
\end{theorem}

However Gallager's result is valid for only XOR queries. We can provide a lower bound that is close to Gallager's bound, and holds for any type of query function.
\begin{theorem}
\label{thm:2}
Assume a  label vector $\bfX$ with each entry being i.i.d. $\textup{Ber}(p), p < \frac12$. The minimum number of queries, each involving  at most $\Delta$ elements, necessary to  recover all labels with high probability is at least by $n H(p)\cdot\max\{ 1, \max_{\rho} \frac{(1-\rho)}{H(\frac{(1-\rho )r(p)\Delta}{\rho})}\} $ where $r(p)\equiv2p(1-p).$ 
\end{theorem}
\begin{proof}
Every query involves at most $\Delta$ elements. Therefore the average number of queries an element is part of is 
$\frac{\Delta m}{n}$. Therefore $1-\rho$ fraction of all the elements (say the set $S \subset \{1, \dots , n\}$) are part of less than $\frac{\Delta m}{\rho n}$ queries. Now consider the set $\{1, \dots , n\} \setminus S$. Consider all typical label vectors $\cC\in \{0,1\}^n$ such that their projection on $\{1, \dots , n\} \setminus S$ is a fixed typical sequence. We know that there are $2^{n(1-\rho)H(p)}$ such sequences. Let $\bfX_0$ be one of these sequences. Now, almost all sequences of $\cC$ must  have a distance of $n(1-\rho )r(p)+o(n)$  from $\bfX_{0}$. Let $\bfY_0$ be the corresponding query outputs when $\bfX_0$ is the input. Now any query output for input belonging to $\cC$ must reside in a Hamming ball of radius $\dfrac{(1-\rho )r(p)\Delta m}{\rho}$ from $\bfY_{0}$. Therefore, comparing the volume of the balls, we must have 
$ mH(\frac{(1-\rho )r(p)\Delta}{\rho}) \geqslant n(1-\rho)H(p).$
\end{proof}


\begin{figure}
\centering
       \includegraphics[width=0.55\textwidth]{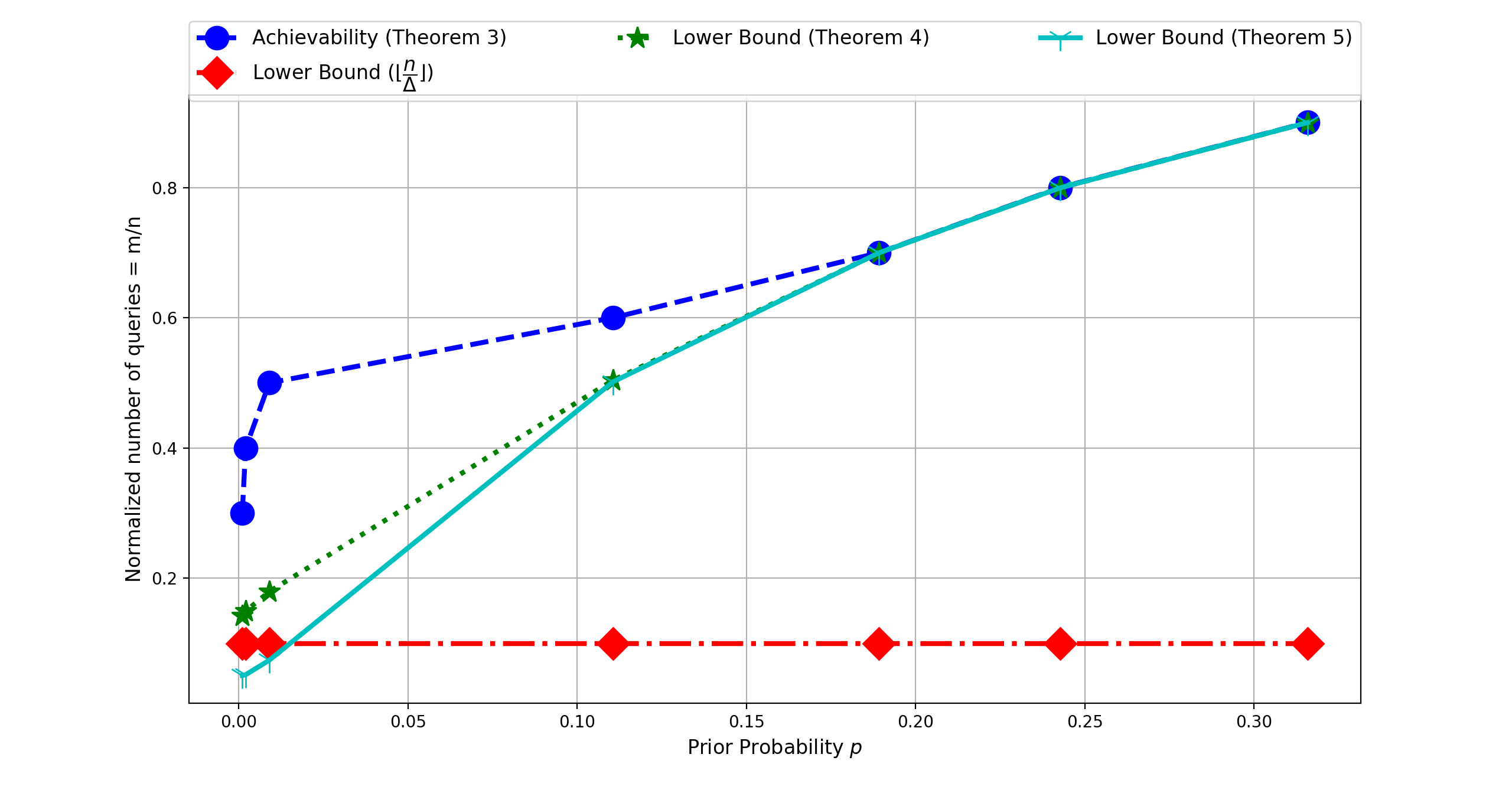}
    \caption{Normalized Query complexity $\frac{m}{n}$ of Theorem~\ref{thm:1} (blue) as a function of prior probability $p$, for $\Delta=10$. \label{fig:first}}
\end{figure}


Finally, in Figure \ref{fig:first}, we have compared the achievability scheme (Theorem \ref{thm:1}) and the lower bounds presented above for $\Delta=10$. For larger $\Delta$ the bounds are even closer. 

\subsection{Noiseless queries and approximate recovery}
\label{subsec:Perfect+Approximate}
Again let us consider the case when $k=2$, i.e., only two possible labels.
Let $\bfX\in \{0,1\}^n$ be the label vector. Suppose we have a querying algorithm that, by using $m$ queries, recovers a label vector $\hat{\bfX}$. \begin{definition}We call a querying algorithm to be $(1-\delta)$-good if for any label vector, at least $(1-\delta) n$ labels are correctly recovered. If the querying algorithm is randomized then we want at least $(1-\delta) n$ labels in expectation to be correctly recovered. This means for any label and recovered label pair $\bfX, \hat{\bfX}$, the Hamming distance is at most $\delta n$. For an  equivalent definition, we can define a distortion function $d(X, \hat{X}) = X+ \hat{X} \mod 2$, for any two labels $X, \hat{X}\in \{0,1\}$. We can see that $\avg d(X, \hat{X}) = \Pr(X\neq \hat{X})$, which we want to be  bounded by $\delta$.\end{definition}
%
%
Using standard rate-distortion theory \cite{cover2012elements}, it can be seen that, if the queries could involve an arbitrary number of elements then with $m$ queries it is possible to have a $(1-\tilde{\delta}(m/n))$-good scheme where $\tilde{\delta}(\gamma) \equiv H^{-1}(H(p) - \gamma)$. 
When each query is allowed to take only at most $\Delta$ inputs, we have the following lower bound for $(1-\delta)$-good querying algorithms.

\begin{theorem}
\label{thm:3}
In any $(1-\delta)$-good querying scheme with $m$ queries where each query can take as input $\Delta$ elements, the following must be satisfied (below $h'(x) = \frac{dH(x)}{dx}$):
$$
\delta \ge \tilde{\delta}\big(\frac{m}{n}\big)+\frac{H(p)-H(\tilde{\delta}(\frac{m}{n}))}{h'(\tilde{\delta}(\frac{m}{n}))(1+e^{\Delta h'(\tilde{\delta}(\frac{m}{n}))})}.
$$
%
\end{theorem}
The proof of this theorem is somewhat involved, and we have provided it in Section~\ref{sec:thm3}.


One of the main observation that we make  is that the `same cluster' queries are highly inefficient for approximate recovery. This follows from a classical result of Ancheta and Massey \cite{massey1977joint} on the limitation of linear codes as rate-distortion codes. Recall that, the `same cluster' queries are equivalent to XOR operation in the binary field, which is a linear operation on $GF(2)$. We rephrase a conjecture by Massey in our terminology.
 \begin{conjecture}[`same cluster' query lower bound]
 For any $(1-\delta)$-good scheme with $m$ `same-cluster' queries ($\Delta =2$), the following must be satisfied:
$
\delta \ge p(1-\frac{m}{nH(p)}).
$
\end{conjecture}
This conjecture is known to be true  at the point $p = 0.5$ (equal sized clusters). We have plotted these two lower bounds in Figure \ref{fig:Rate0.5}. 

With `same-cluster' queries, the following nonadaptive querying scheme for approximate recovery matches the above conjecture for $p =0.5$. 

\begin{theorem}
There exists an $(1-\delta)$-good scheme with $m$ `same-cluster' queries ($\Delta =2$), with:
$
\delta = p(1-\frac{m}{n}).
$
\end{theorem}
\begin{proof}
Without loss of generality, let us assume $p < \frac12$. As discussed earlier, for $p =\frac12$, instead of labels a clustering can be obtained with minor changes in the scheme. Just consider the scheme that just compares $(1-\delta/p)n$ elements with the first element. 
 Set the label of the rest of the $n\delta/p$ elements to be 0. By following the same arguments as in Theorem~\ref{thm:-1}, the labels of the first $(1-\delta/p)n$ elements can be recovered exactly. Among the rest of the elements with high probability at most $\delta n +o(n)$ labels will be wrong. Therefore this is an $(1-\delta)$-good scheme with $m = (1-\delta/p)n$.
\end{proof}


\begin{figure}[htbp]
\centering
 \includegraphics[scale =0.5]{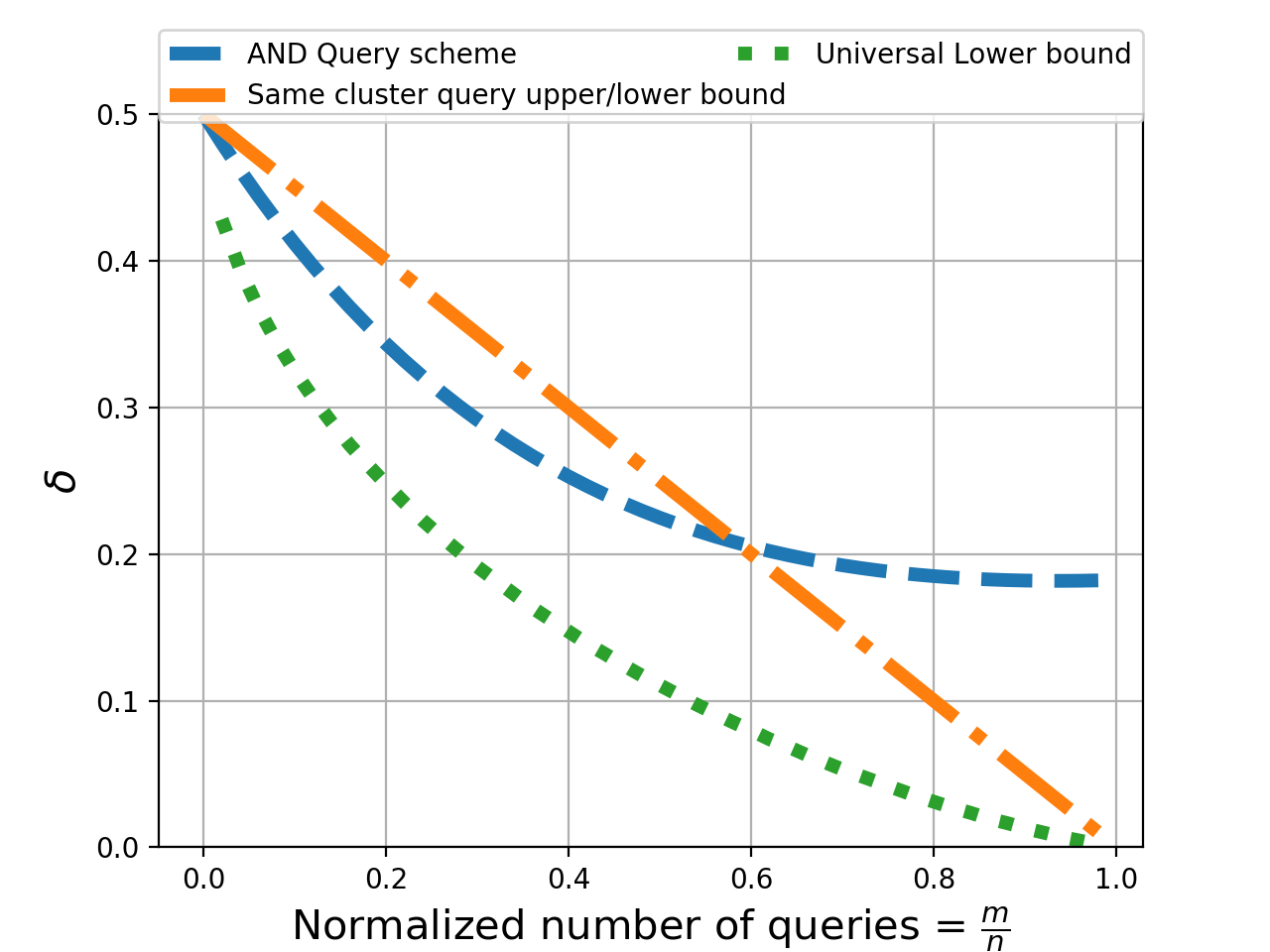}
    \caption{Performance of $(1-\delta)$-good schemes with noiseless queries; $p=0.5$. The lower bound refers to Theorem~\ref{thm:3}.\label{fig:Rate0.5}}
\end{figure}

With that, let us now provide a querying scheme with $\Delta =2$ that will provably be better than `same-cluster' schemes for a large regime of $\delta$.
 
\paragraph{Querying scheme: AND queries}  We define the AND query $Q:\{0,1\}^{2} \rightarrow \{0,1\}$ as $Q(X,X')=X \bigwedge X',$ where $X,X'\in \{0,1\},$ so that $Q(X,X')=1$ only when both the elements have labels equal to $1$. 
For each pairwise query  the oracle will return this AND operation of the labels.
\begin{theorem}
\label{thm:4}
There exists a $(1-\delta)$-good querying scheme with $m$ pairwise AND queries such that
\begin{align*}
\delta &=p(1-2/n)^m+\sum_{d=1}^{m} \binom{m}{d}(2/n)^d(1-2/n)^{m-d}\\
& \times \sum_{l=1}^{d} {n \choose l}\frac{f(l,d)}{n^{d}}(1-p)^{l}p
\end{align*}
where  
$f(l,d)=\sum_{i=0}^l (-1)^i \binom{l}{i} (l-i)^d.$  
\end{theorem}
\begin{proof}
Assume $p<0.5$ without loss of generality. Consider a random bipartite graph where each `left' node represent an element labeled according to the label vector $\bfX\in \{0,1\}^n$ and each `right' node represents a query. All the query answers are collected in  $\bfY\in \{0,1\}^m$. The graph has right-degree exactly equal to 2. For each query the two inputs are selected uniformly at random without replacement. 

{\em Recovery algorithm:} For each element we look at the queries that involves it and estimate its label as $1$ if any of the query answers is $1$ and predict $0$ otherwise. If there are no queries that involves the element, we simply output 0 as the label.

 Since the average left-degree is $\frac{2m}{n}$ and since all the edges from the right nodes are randomly and independently thrown,  the degree of each left-vertex is distributed according to a Binomial distribution with the mean $\lambda=\frac{2m}{n}$. We define element ${j}$ to be a two-hop-neighbor of ${i}$ if there is at least one query which involved both the elements ${i}$ and ${j}$ . Under our decoding scheme we only have an error when the label of $i$, $X_{i}=1$ and the labels of all its two-hop-neighbors are $0$. Hence the probability of error for estimating $X_{i}$ can be 
 written as, 
 $
 \Pr(X_i \ne \hat{X}_i) = \sum_d \Pr ({\rm degree}(i) = d) \Pr(X_i \ne \hat{X}_i \mid {\rm degree}(i) = d).
 $ 
 Now let us estimate $\Pr(X_i \ne \hat{X}_i \mid {\rm degree}(i) = d)$.
 We further condition the error on the event that there are $l$ distinct two-hop-neighbors (lets call the number of distinct neighbors of ${i}$ as ${\rm Dist}({i})$) and hence we have that $\Pr(X_i \ne \hat{X}_i \mid {\rm degree}(i) = d) = \sum_{l=1}^{d}\Pr({\rm Dist}({i})=l)\Pr(X_i \ne \hat{X}_i | {\rm degree}(i)=d,{\rm Dist}({i})=l) = \sum_{l=1}^{d} {n \choose l} \frac{f(l,d)}{n^{d}}p(1-p)^l$. Now using the fact that the degree distribution is Binomial($m, 2/n$) we get the statement of the theorem.
%
%
\end{proof}
The performance of this querying scheme is plotted against the number of queries  for prior probabilities $p=0.5$ 
in Figure \ref{fig:Rate0.5}. 
We see in  Figure \ref{fig:Rate0.5} that the AND query scheme beats the `same cluster' query lower bound for a range of query-performance trade-off in approximate recovery for $p =0.5$.  
 Similar trend is exhibited for smaller $p$, but the lower and upper bounds for `same cluster' queries diverge.   
\subsection{Noisy queries and approximate recovery}
\label{subsec:Faulty+Approximate}
This section contains our main algorithmic contribution. In contrast to the previous sections here we consider 
the general case of $k\ge 2$ clusters. Recall that the label vector $\bfX\in \{0,1, \dots, k-1\}^n$, and 
the prior probability of each label is given by the probabilities $\bfp = (p_0, \dots, p_{k-1})$. Let $p^\ast\equiv \min_i p_i$ be the {\em relative size of the smallest cluster}. Assume a model of noise in the oracle answer, where each answer is wrong independently with probability $q$.
In this section we study two types of queries, the `same-cluster' queries and a generalization of the AND queries. There are two reasons to study AND queries in this setting. First, we experimentally observe that the error-rate of response to AND queries is less than that of `same-cluster' queries; second, as we will see, theoretically AND queries have a better resilience to noise at the expense of being `more informative'. 

 Note that we do not allow the same query to be asked to the oracle multiple time. 
 As discussed earlier, with only `same-cluster' queries, it is only possible to resolve the labeling problem up to a permutation of labels; i.e., with the `same-cluster' queries we solve only the clustering problem. 
A $(1-\delta)$-good approximation scheme for labeling was defined before. For a clustering, a (randomized) scheme will be called $(1-\delta)$-good if (on expectation) at most $\delta$ proportion of all elements are assigned wrong clusters.

\paragraph{Same Cluster Queries:} Schemes with `same-cluster' queries in the nonadaptive setting have  been studied in \cite{mazumdar2017clustering}. Albeit \cite{mazumdar2017clustering} focuses on exact recovery with high probability instead of approximate recovery, we can modify the schemes presented therein to obtain results in our precise setting. Let us define, for a pair of elements $u,v \in \{1,\dots, n\}$, the same cluster query answers $Y_{u,v} \in  \{0,1\}$, a Bernoulli random variable. $Y_{u,v}$ is the correct answer to the `same-cluster' query between elements $u$ and $v$ with probability $1-q$ and the incorrect answer with probability $q$. 
The querying scheme is presented in Algorithm~\ref{alg:same_cluster}.  In the algorithm and its analysis, $\mathsf{XOR}:\{0,1\}^\ast \to \{0,1\}$ simply denotes the XOR function between the inputs.  At the high level,  the algorithm proceeds by randomly sampling a subset of elements.  Every pairwise `same-cluster' query within that  subset is performed. We show that it is possible to obtain the correct labeling of all elements in the subset, as long as the noise probability is small enough. Each of the elements of the subset is also compared with every element out of the subset to figure out the labels of the rest of the elements. With this algorithm, we have the following result.

\begin{algorithm} [t]                   
\caption{Noisy query approximate recovery with `same-cluster' queries}          
\label{alg:same_cluster}                           
\begin{algorithmic}                    
    \REQUIRE PRIOR $\bfp \equiv (p_0, \dots, p_{k-1})$, Noise $q$, Randomly chosen (without replacement) subset $\mathcal{S} \subset \{1, \dots, n\}$ of elements. 
    \REQUIRE Query Answers $Y_{u,v}: (u,v) \in \mathcal{S}$ and $Y_{u,v}: u \in \mathcal{S}, v\in  \{1, \dots, n\} \setminus \mathcal{S}$
     \FOR{$u,v \in \mathcal{S}, u \neq v$}
        \STATE Compute $Z_{u,v} = \sum_{z \in \mathcal{S}\setminus \{u,v\}} \mathsf{XOR}(Y_{u,z},Y_{v,z})$
     \ENDFOR 
    \STATE Set a threshold $\theta=(|\mathcal{S}|-2)(2q(1-q)+\min_{i,j \in \{0,\dots,k-1\}, i\neq j}(p_i+p_j)(1-2q)^2/2)$
    \STATE Form a graph with $\mathcal{S}$ being the set of vertices, with  $(u,v)$ being an edge if $Z_{u,v} \le \theta$. Find the connected components in the graph: $\mathcal{C}_0,\mathcal{C}_1,\dots,\mathcal{C}_{k'-1}$.
    \FOR{$u \in [n]\setminus \mathcal{S}$}
         \FOR{$i \in  \{0, 1, \dots, k'-1\}$}
             \IF{$\sum_{v \in \mathcal{C}_i} Y_{u,v} \ge |\mathcal{C}_i|/2$}
                  \STATE Assign $\mathcal{C}_i = \mathcal{C}_i  \cup \{u\}$.
             \ENDIF    
          \ENDFOR   
    \ENDFOR           
\end{algorithmic}
\end{algorithm}

\begin{theorem}
\label{thm:45}
 The querying scheme of Algorithm \ref{alg:same_cluster} with 
\begin{align*}
m &=O\big( \frac{n(q(1-q)+p^\ast(1-2q)^2)}{(p^\ast(1-2q)^2)^2} \\ 
&\cdot\log\frac{q(1-q)+p^\ast(1-2q)^2}{\delta(p^\ast(1-2q)^2)^2}\big)
\end{align*} 

`same cluster' queries is a $(1-\delta)$-good for approximate recovery of the correct clustering from noisy queries. In particular, for a constant noise probability $0 <q<1/2$, $O(\frac{n}{{p^\ast}^2} \log \frac1{\delta p^\ast})$ `same-cluster' queries suffice  for $(1-\delta)$-good recovery.
\end{theorem}

We will prove Theorem \ref{thm:45} using the following lemmas.

\begin{lemma}\label{lem:hel}
For a subset $\mathcal{S} \subseteq [n]$ of all the elements, if we perform all the possible ${|\mathcal{S}| \choose 2}$ `same-cluster' queries, then we can exactly recover a clustering of all the elements of $\mathcal{S}$ according to their labels, provided 
\begin{align*}
|\mathcal{S}|  &\ge \frac{c(q(1-q)+p^\ast(1-2q)^2)}{(p^\ast(1-2q)^2)^2} \\
&\log\frac{q(1-q)+p^\ast(1-2q)^2}{\delta(p^\ast(1-2q)^2)^2},
\end{align*}
with probability at least $1-\delta/2$ for some absolute constant $c$. 
\end{lemma}

\begin{proof}  
Consider two distinct elements $u, v \in \mathcal{S}$. Consider the two possible hypotheses: $H_1$ (both $u$ and $v$ have the same label i.e $X_u=X_v$) and $H_2$ ($u$ and $v$ have different labels i.e. $X_u \neq X_v$). 
We have for any element $z \in \{1, \dots, n\} \setminus \{u,v\}$,
\begin{align*}
\Pr(\mathsf{XOR}(Y_{u,z},Y_{v,z}) = 1 \mid H_1) = 2q(1-q).
\end{align*}     
On the other hand, if $X_u =i$ and $X_v =j$ and $i \ne j$,
\begin{align*}
&\Pr(\mathsf{XOR}(Y_{u,z},Y_{v,z})  = 1 \mid H_2) \\
&= (1-p_i-p_j)\cdot 2q(1-q)+(p_i+p_j)(q^2+(1-q)^2) \\
&=2q(1-q)+(p_i+p_j)(1-2q)^{2}. 
\end{align*}
Therefore, in order to determine which hypothesis is true, we will consider the queries of $u,v$ with all labels $z \in \mathcal{S}\setminus \{u,v\}$. 
In that case, 
\begin{align*}
&\mathbb{E} \left[ \sum_{z \in \mathcal{S}\setminus \{u,v\}}  \mathsf{XOR}(Y_{u,z},Y_{v,z})\mid H_1 \right]  \\
&= 2(|\mathcal{S}|-2)q(1-q); \\
&\mathbb{E} \left[ \sum_{z \in \mathcal{S}\setminus \{u,v\}}  \mathsf{XOR}(Y_{u,z},Y_{v,z})  \mid H_2 \right] \\
&= 2(|\mathcal{S}|-2)q(1-q)+(|\mathcal{S}|-2)(p_i+p_j)(1-2q)^2. \\
\end{align*}
Therefore, if the deviation from the mean of the statistic $\sum_{z \in \mathcal{S}\setminus \{u,v\}}  \mathsf{XOR}(Y_{u,z},Y_{v,z}) $ is at most $(|\mathcal{S}|-2)\cdot \min_{i,j \mid i \neq j}(p_i+p_j)(1-2q)^2/2$, then we can correctly infer which one of $H_1,H_2$ is true. Using Chernoff bound, we can upper bound the probability of deviation by: 
\begin{align*}
&\Pr \Big(\sum_{z \in \mathcal{S}\setminus \{x,y\}}  \mathsf{XOR}(Y_{u,z},Y_{v,z})- \\
& \mathbb{E} \sum_{z \in \mathcal{S}\setminus \{u,v\}}  \mathsf{XOR}(Y_{u,z},Y_{v,z}) \ge \epsilon (|\mathcal{S}|-2)\Big| H_1 \Big) \\
&\le e^{-\frac{(|S|-2)\epsilon^2}{6q(1-q)}}, 
\end{align*}
and
\begin{align*}
&\Pr\Big(\sum_{z \in \mathcal{S}\setminus \{x,y\}}  \mathsf{XOR}(Y_{u,z},Y_{v,z})-  \\
&\mathbb{E} \sum_{z \in \mathcal{S}\setminus \{u,v\}}  \mathsf{XOR}(Y_{u,z},Y_{v,z}) \le \epsilon (|\mathcal{S}|-2)\Big| H_2 \Big) \\
&\le e^{-\frac{(|S|-2)\epsilon^2}{4(q(1-q)+\epsilon)}},
\end{align*}
where $\epsilon=\min_{i,j \mid i \neq j}(p_i+p_j)(1-2q)^2/2$.  In Algorithm \ref{alg:same_cluster}, we infer whether two elements are in same or different clusters at most $|\mathcal{S}|^2$ times. If we set,
$$
|\mathcal{S}|=O\Big(\frac{q(1-q)+\epsilon}{\epsilon^2}\log\frac{q(1-q)+\epsilon}{\delta\epsilon^2}\Big),
$$
then by the union bound, the probability that a single pair of elements are misclassified is at most $\delta/2$.
Plugging in the value of $\epsilon$ we get that the value promised in the lemma
suffices so that no pair is misclassified with probability at least $1-\delta/2$.
\end{proof}

\begin{lemma}{\label{lem:helper1}}
In a subset $\mathcal{S} \subseteq \{1, \dots, n\}$ of randomly selected elements, the number of elements with label $i$, for all $i\in \{0,1, \dots, k-1\}$, is at least  $|\mathcal{S}|p_i/2$ with probability at least $1-k\exp(-|\mathcal{S}|{p^\ast}/8)$.
\end{lemma}
\begin{proof}
The expected size of a cluster with label $j$ is $|\mathcal{S}|p_j$ and using Chernoff bound, the size of the cluster is at least $|\mathcal{S}|p_j/2$ with probability at least $1-\exp(-|\mathcal{S}|p_j/8)$. Taking a union bound over all the $k$ labels, we have the proof of the lemma.
\end{proof}

We are now ready to prove Theorem \ref{thm:45}.

\begin{proof}[Proof of Theorem \ref{thm:45}]
We follow the steps of the algorithm (Algorithm~\ref{alg:same_cluster}).
We first randomly select a subset $|\mathcal{S}|$ of elements satisfying the condition of Lemma~\ref{lem:hel}. 
Therefore, using Lemma~\ref{lem:hel}, Algorithm~\ref{alg:same_cluster} recovers a perfect clustering of $\mathcal{S}$ in to $\mathcal{C}_0, \dots, \mathcal{C}_{k-1}$ with probability at least $1-\delta/2$. Furthermore, using Lemma~\ref{lem:helper1}, we know that $|\mathcal{C}_i| \ge |\mathcal{S}|p_i/2$ for all $i$ with probability at least $1-k\exp(-|\mathcal{S}|{p^\ast}/8) \ge 1 - \frac{1}{p^\ast} \exp(-|\mathcal{S}|{p^\ast}/8) \ge 1 -\frac{\delta}4,$ since $p^\ast \le 1/k$, and by substituting the value of $|\cS|$.

Now, for an element $u \in \{1, \dots, n\} \setminus \mathcal{S}$, consider the sum $\sum_{v \in \cC_i} Y_{u,v}$. Consider the following two cases.

When, $u$ has the same label as the elements in $\cC_i$,
$$
\Pr(\sum_{v \in \cC_i} Y_{u,v} \le |\cC_i|/2) \le \exp(-|\cC_i| \frac{(1-2q)^2}{8(1-q)}),
$$
and, when $u$ does not have the same label as the elements in $\cC_i$,
$$
\Pr(\sum_{v \in \cC_i} Y_{u,v} \ge |\cC_i|/2)  \le \exp(-|\cC_i| \frac{(1-2q)^2}{12q}).
$$

Therefore, given $|\mathcal{C}_i| \ge |\mathcal{S}|p_i/2$ for all $i$, the probability that $u \in \{1, \dots, n\} \setminus \mathcal{S}$ will be included in an  incorrect cluster is at most $2k\exp(-c'|\cS|p^\ast \frac{(1-2q)^2}{\max(q,1-q)}) \le \delta/4$, by substituting  for the value of $|\cS|$ and again noting that $p^\ast \le 1/k.$


Therefor the expected number of elements that are misclassified is at most
\begin{align*}
\small
& n \cdot \Pr(\text{clustering of $\mathcal{S}$ is incorrect}) \\
                                           &+n \cdot\Pr(\text{inference of label is incorrect} \mid \\
                                           & \text{clustering of $\mathcal{S}$ is correct}) 
                                            = n(\delta/2+\delta/4+\delta/4) = n\delta.                                           
\end{align*}
Thus the scheme is $(1-\delta)$-good with
 the total number of queries being used is $O(|\mathcal{S}|^2+n|\mathcal{S}|)$.
\end{proof}

\paragraph{AND Querying Scheme} 
Note that, with `same-cluster' queries, we could only recover the clustering, and not the labeling of the elements. In this subsection, we will recover an approximate labeling of elements. Also, instead of binary output queries, in this part we consider an oracle that can provide one of $k$ different answers. 
We consider a model of noise in the query answer where the oracle provides correct answer with probability $1-q$, and any one of the remaining incorrect answers with probability $\frac{q}{k-1}$ 

We only perform pairwise queries. For a pair of labels $X$ and $X'$ we define a query $Y=Q(X,X') \in  \{0,1,\dots,k-1\}$.  
For our algorithm we define the $Q$ as 
 \[
    Q(X,X')= \left\{\begin{array}{llr}
        i & \textit{if } & X=X' =i \\
        0 &  &\text{otherwise. }\\
        \end{array}\right\} 
  \]  
We can observe that for $k=2$, this query is exactly same as the binary AND query that we defined in the previous section.  
In the general setting, it is equivalent to asking `what is the common label of the two elements?'.
In our querying scheme, we make a total of $\frac{nd}{2}$ queries, for an integer $d>1.$ We design a $d$-regular graph $G(V,E)$ where $V= \{1, \dots, n\}$ is the set of elements that we need to label. We query  all the pairs of elements $(u,v)\in E$. For any $u \in V$ let us define $\cN(u) = \{v\in V: (u,v) \in E\}$ to be the neighborhood of $u$.

Under this querying scheme, we propose to use Algorithm \ref{alg1} for reconstructions of labels.

\begin{algorithm} [t]                   
\caption{Noisy query approximate recovery with $\frac{nd}{2}$ queries}          
\label{alg1}                           
\begin{algorithmic}                    
    \REQUIRE PRIOR $\bfp \equiv (p_0, \dots, p_{k-1})$, {\color{black} $G(V,E), q$}
    \REQUIRE Query Answers $Y_{u,v}: (u,v) \in E$
    \FOR{$i\in[1,2,\dots,k-1]$}
        \STATE $C_{i}=\frac{dq}{k-1}+\frac{dp_{i}}{2}\big(1-\frac{qk}{k-1}\big)$ 
    \ENDFOR
    \FOR{$u \in V$}
       \FOR{$i\in \{1,2,\dots,k-1\}$}
           \STATE $N_{u,i} \equiv \sum_{v \in \cN(u)} \mathds{1}\{Y_{u,v}=i\}$
           \IF{$N_{u,i} \ge \lceil C_{i}\rceil$}
                \STATE $X_{u} \leftarrow i$
                 \STATE Assigned $\leftarrow$ True
                \STATE \textbf{ break}
            \ENDIF        
       \ENDFOR
       \IF{$\neg$ Assigned}
                \STATE $X_{u} \leftarrow 0 $
       \ENDIF                
    \ENDFOR    
\end{algorithmic}
\end{algorithm}

\begin{theorem}
\label{thm:5}
 The querying scheme with  $m = O\Big(\frac{n}{p^\ast(1-\frac{qk}{k-1})}\max\big(\frac{q}{k (1-\frac{qk}{k-1})p^\ast },1\big)\log\frac{k}{\delta}\Big)$ AND queries  as above and Algorithm \ref{alg1} is $(1-\delta)$-good for approximate recovery of labels from noisy queries. In particular, if $p^\ast \sim \frac{1}{k}$, then, for a constant noise probability $0 <q < 1-\frac{1}{k}$, $O(\frac{n}{p^\ast}\log \frac{k}{\delta})$queries suffice  for $(1-\delta)$-good recovery.

\end{theorem}
\begin{proof} 
The total number of queries is $m =\frac{nd}{2}$. 
Now for a particular element $u \in V$,  we look at the values of $d$ noisy oracle answers $\{Y_{u,v}: v \in \cN(u)\}$. 
Consider the two cases.

Case 1: 
the true label of $u$ is $i\in \{1, \dots, k-1\}$: 

We have,
\begin{align*}
\Pr(Y_{u,v} =i) &= p_i(1-q) + (1-p_i)\frac{q}{k-1} \\
&= \frac{q}{k-1} +p_i\big(1-\frac{qk}{k-1}\big).
\end{align*}
Therefore, $\avg(N_{u,i}) = \frac{dq}{k-1}+dp_{i}\big(1-\frac{qk}{k-1}\big)$.
Using Chernoff bound, 
$$
\Pr(N_{u,i} < C_{i}) = \exp\Big(-\frac{d}{2}\cdot\frac{(\frac{p_i}{2}(1-\frac{qk}{k-1}))^2}{\frac{q}{k-1}+p_i(1-\frac{qk}{k-1})}\Big) \le \frac{\delta}{2},
$$
as long as 
$d = c\cdot \frac{\frac{q}{k-1}+p_i(1-\frac{qk}{k-1})}{(\frac{p_i}{2}(1-\frac{qk}{k-1}))^2} \log \frac1\delta$ 
for some constant $c$.

Case 2: the true label of $u$ is not $i$, i.e., from $\{0,1, \dots, i-1, i+1, \dots, k-1\}$:

Then, 
$$
\Pr(Y_{u,v} =i) = \frac{q}{k-1}.
$$
Hence, $\avg(N_{u,i}) = \frac{dq}{k-1}$, and using Chernoff bound,
$$
\Pr(N_{u,i} \ge C_{i}) \le \exp\Big(-\frac{d}{3}\frac{(\frac{p_i}{2}(1-\frac{qk}{k-1}))^2}{\frac{q}{k-1}}\Big) \le \frac{\delta}{2(k-1)},
$$
as long as $d = c' \cdot \frac{\frac{q}{k-1}}{(\frac{p_i}{2}(1-\frac{qk}{k-1}))^2} \log \frac{k}{\delta}$ for some constant $c'$.

Choosing, $d = c'' \cdot \frac{\frac{q}{k-1}+p_i(1-\frac{qk}{k-1})}{(\frac{p_i}{2}(1-\frac{qk}{k-1}))^2} \log \frac{k}\delta$ will ensure that, if the label of an element is  $i \in\{1, \dots, k-1\}$, then the probability that the element is mislabeled is at most $\frac{\delta}{2} + (k-1)\frac{\delta}{2(k-1)} =\delta$ (by the union bound). Therefore, we can choose,
$$
d = O\Big(\frac{\frac{q}{k-1}+p^\ast(1-\frac{qk}{k-1})}{(\frac{p^\ast}{2}(1-\frac{qk}{k-1}))^2} \log \frac{k}\delta\Big).
$$
Note that, if the true label of the element is $0$, even then the probability that it is being mislabeled as something else is at most $\delta$, since we only assign an element label $0$, if it is not assigned any other labels. 


Therefore, it will suffice to choose,
\begin{align*}
d &= O\Big(\frac{\max(\frac{q}{k-1} ,p^\ast(1-\frac{qk}{k-1}))}{(p^\ast(1-\frac{qk}{k-1}))^2} \log \frac{k}\delta\Big)  \\
&= O\Big(\frac{1}{p^\ast(1-\frac{qk}{k-1})}\max\big(\frac{q}{k (1-\frac{qk}{k-1})p^\ast },1\big)\log\frac{k}{\delta}\Big).
\end{align*}

Since the total number of queries is $\frac{nd}{2}$ we have the claim of the theorem.

In addition, it can also be shown that the number of incorrect labels is $\delta n$ with high probability. Let $Z_u$ be the event that element $u$ has been incorrectly labeled. Then $\avg Z_u = \delta$. The total number of incorrectly labeled elements is $Z= \sum_u Z_u$. We have $\avg Z = n \delta$.   Now define $Z_{u} \sim Z_{v}$ if $Z_{u}$ and $Z_{v}$ are dependent. Now 
$\Delta^{*}\equiv\sum_{Z_{u} \sim Z_{v}}\Pr(Z_{u}|Z_{v}) \le d^{2}+d$ because the maximum number of nodes dependent with $Z_{u}$ are the $1$-hop and $2$-hop neighbors of $u$ in the graph $G$.  Now using Corollary 4.3.5 in \cite{alon2004probabilistic},  it is evident that $Z=\avg Z = n\delta$ almost always.

\end{proof}

\begin{remark}
While not surprising, it is worth mentioning that Algorithm~\ref{alg1} will work if instead of the prior probabilities the sizes of the clusters are known.  The ground truth clusters can be adversarial as long as they maintain the relative sizes.
\end{remark}

The theoretical performance guarantee of Algorithm \ref{alg1} 
for $k=2$ is shown in Figures \ref{fig:Varq} and \ref{fig:Vard}.  
We can observe from Figure \ref{fig:Varq}  that for a particular $q$, incorrect labeling  decreases as $p$ becomes higher. We can also observe from  Figure \ref{fig:Vard} that if $q=0.5$ then the incorrect labeling is 50\% because the complete information from the oracle is lost. For other values of $q$, we can see that the incorrect labeling decreases with increasing $d$. 


\paragraph{Comparing `same-cluster' queries and AND queries.}
In this section we presented two algorithms for approximate recovery, one based on `same-cluster' query, the other based on a generalization of AND queries. Each of these has their advantages. However, for $k=2$ in particular, we found experimentally that for the same dataset, the error-rate in crowd-answers of `same cluster' queries is more than that of AND queries\footnote{In this particular dataset, for the purpose clustering a set of movies in to two parts, we ask the questions ``are movies A and B same genre?'' vs ``are both movies A and B action movies''. }. For $k=2$ this does make sense, as a person can just be familiar with one of the pair of elements, and still answer the AND query correctly.  In  addition, AND queries present less ambiguity  to the crowd. These two positives seem to overcome the demand of more expertise of the crowd worker. Furthermore, AND queries lead to a direct labeling of the elements, instead of a clustering.  

For general $k$, from the expressions of query complexities in Theorems~\ref{thm:5} and  \ref{thm:45}, it can be seen that for the regime where the clusters are of proportional sizes, the query complexity of Algorithm \ref{alg:same_cluster} using `same-cluster' queries can be a factor of $k$ larger than that of Algorithm \ref{alg1} using AND queries (query complexities are $O(nk^2\log\frac{k}{\delta})$ and $O(nk\log\frac{k}{\delta})$ respectively). Furthermore, AND queries can handle an error probability up to $1-\frac1k$ whereas `same-cluster' queries work up to an error probability of $\frac12$. On the flip side, this comparison is somewhat unfair, because a `same cluster' query seeks  only 1 bit of information as opposed to $\log k$ bits of information from a generalized AND query. Moreover, AND queries assume the oracle (crowd) to know the labels, whereas only context is enough to answer the `same cluster' queries.

  \begin{figure}[htbp]
  \centering
   \begin{minipage}[t]{0.4\textwidth}
  \includegraphics[scale = 0.4]{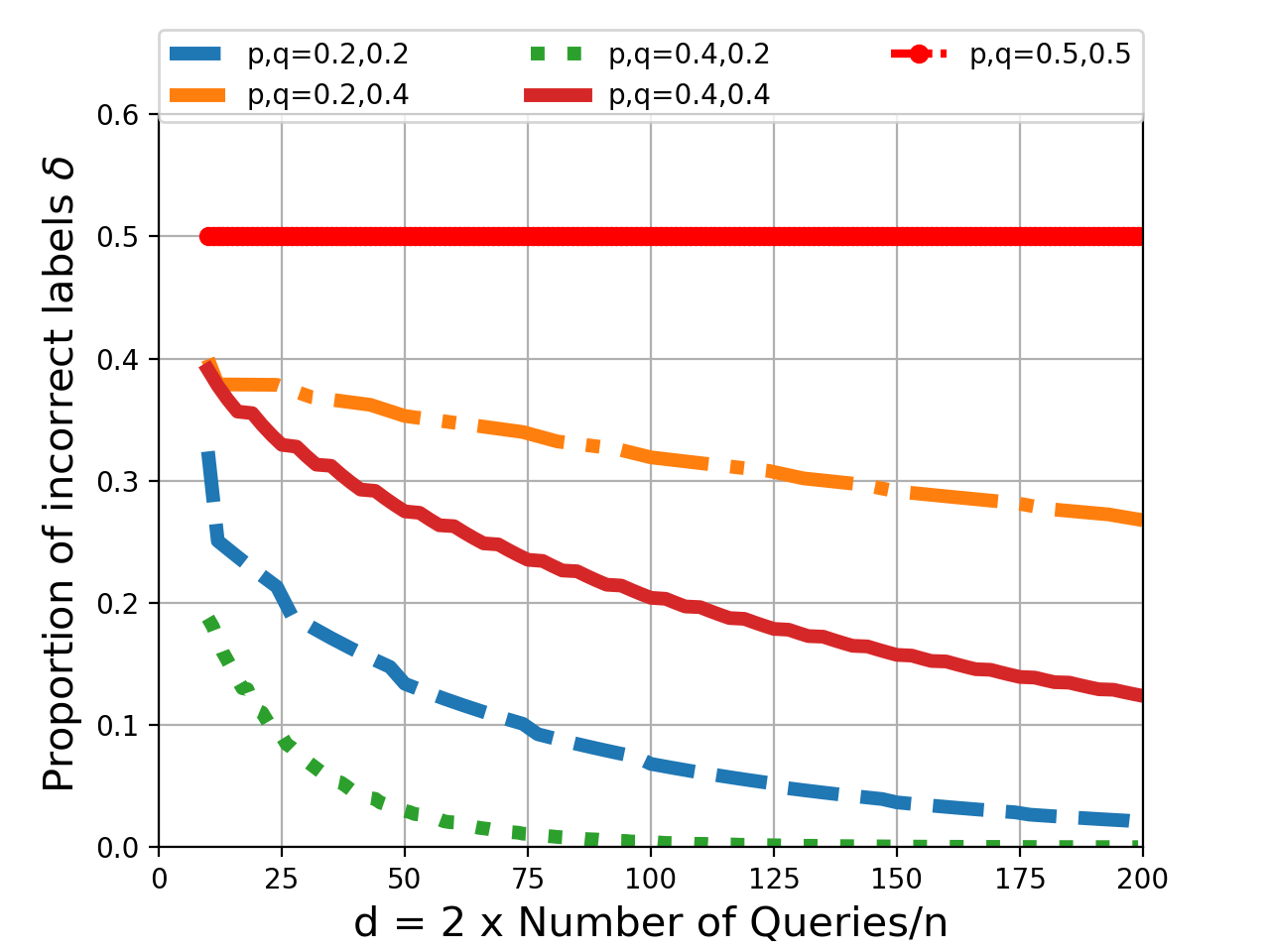}
    \caption{Recovery error for a fixed $p,q$ and varying $d$}
    \label{fig:Vard}
      \end{minipage}
 \hfill
   \begin{minipage}[t]{0.4\textwidth}
    \includegraphics[width=\textwidth]{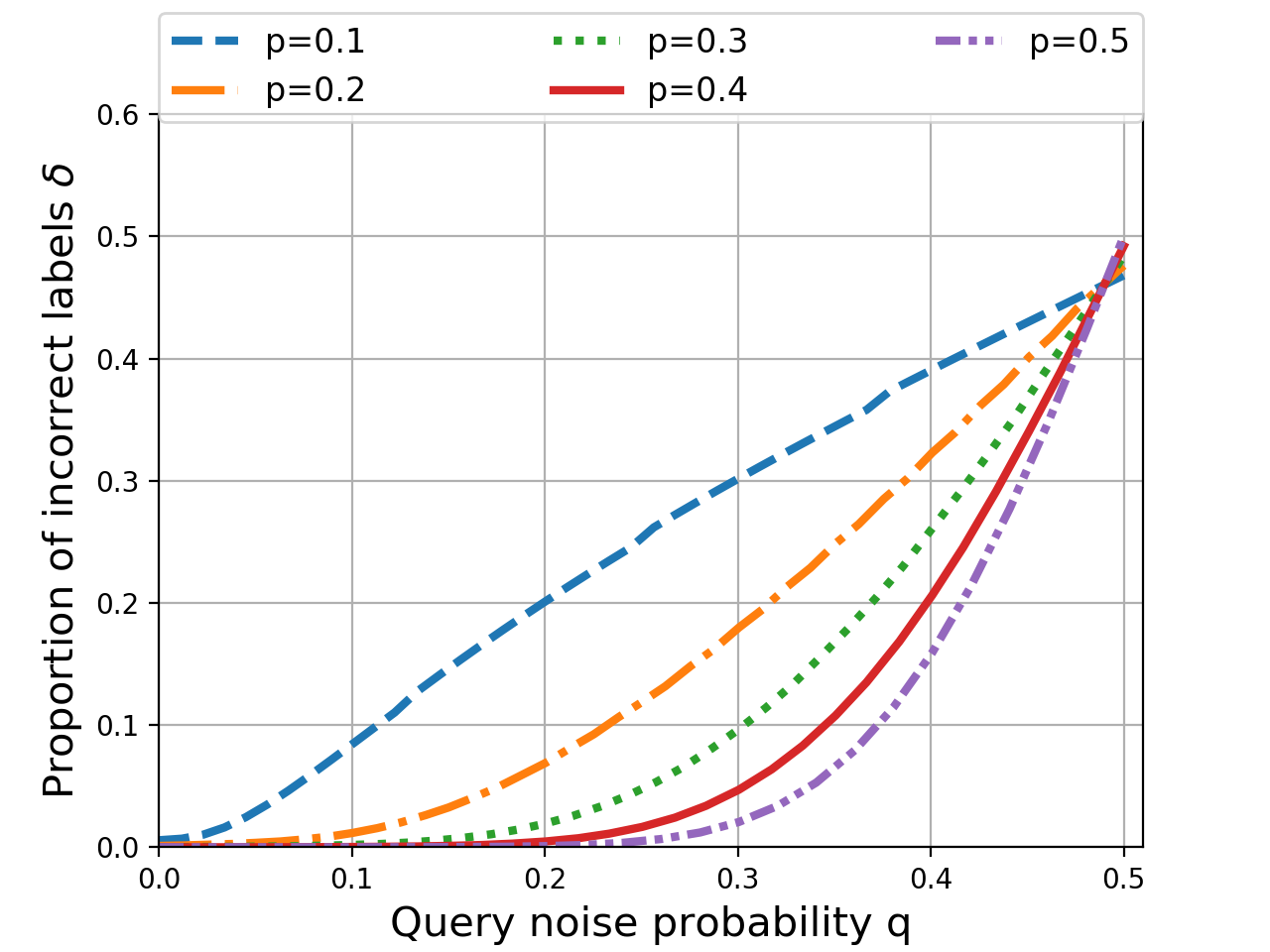}
    \caption{Recovery error for a fixed $p,d=100$ and varying $q$}
    \label{fig:Varq}
  \end{minipage}
  \hfill
   \begin{minipage}[t]{0.4\textwidth}
 \includegraphics[width=\textwidth]{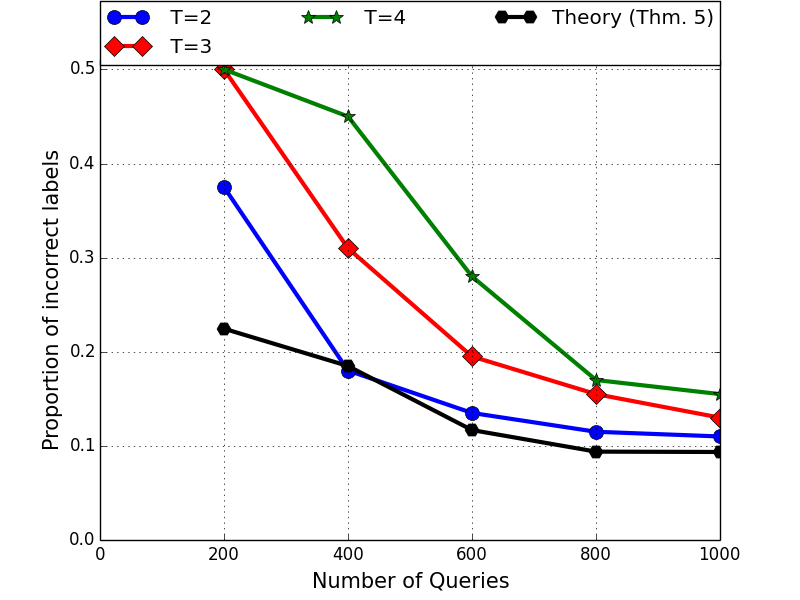}
\caption{Algorithm \ref{alg1} on real crowdsourced dataset  \label{fig:Survey}}  
  \end{minipage}  
\end{figure}

\section{Experiments}
\label{sec:Experiments}
  Though our main contribution is theoretical we have verified our work by using our algorithm on a real dataset created by local crowdsourcing. We first picked a list of 100 `action' movies and 100 `romantic' movies from IMDB \footnote{\texttt{http://www.imdb.com/list/ls076503982/} and \texttt{http://www.imdb.com/list/ls058479560/}}. We then created the queries as described in the querying scheme of Sec.~\ref{subsec:Faulty+Approximate} by defining a $d$-regular graph (where $d$ is even) with the $200$ movies as vertices. To create the graph we put all the movies on a circular list and take a random permutation of the circular list. Then for each node we connected $\frac{d}{2}$ edges on either side to its closest neighbors in the permuted circular list. This random permutation will allow us to use the relative sizes of the clusters as priors as explained in Sec.~\ref{subsec:Faulty+Approximate}. Using $d=10$ , we have $\frac{nd}{2}=1000$ queries with each query being the following question: \textit{Are both the movies   `action' movies?}. Now we divided these 1000 queries into 10 surveys (using SurveyMonkey platform) with each survey carrying 100 queries for the user to answer. We used 10 volunteers to fill up the survey. We instructed them not to check any resources and answer the questions spontaneously and also gave them a time limit of a maximum of 10 minutes. The average finish time of   the surveys were  6 minutes.  The answers represented the noisy query model since some of the answers were wrong.   In total, we have found 105 erroneous answers in those 1000 queries. Now we use Algorithm \ref{alg1} to perform the clustering of movies. For each movie we evaluate the $d$ query answers it is part of,  and use different thresholds $T$ for prediction. That is, for each movie, if there are more than $T$ `yes' answers among those $d$ responses involving that particular movie, we classified the movie as an `action' movie and a `romantic' movie otherwise. 
%
  The  theoretical threshold for predicting an `action' movie is $T= 2$ for oracle error probability $q=0.105,p=0.5$ and $d=10$  but we compared other thresholds as well. 

Now we vary the total number of queries (by tuning $d$) and note the change in prediction accuracy.
We  use Algorithm \ref{alg1} to predict the true label vector from a subset of queries by taking $\tilde{d}$ edges for each node where $\tilde{d}<d$ and $\tilde{d}$ is even i.e $\tilde{d} \in \{2,4,6,8,10\}$. Obviously, for $\tilde{d}=2$ , the thresholds $T=3,4$ are meaningless as we always estimate the movie as `romantic' and hence the distortion starts from $0.5$ in that case. We plotted the error for each case  against the number of queries ($\frac{n\tilde{d}}{2}$) and also plotted the theoretical distortion obtained from our results for $k=2$ labels and $p=0.5,q=0.105$. We compare these results along with the theoretical distortion that we should have for $q=0.105$. All these results have been compiled in Figure \ref{fig:Survey} and we can  observe that the distortion is decreasing with the number of queries and the gap between the theoretical result and the experimental  results is  small for $T=2$. These results  validate our theoretical results and our algorithm to a large extent. 

We have also asked `same cluster' queries with the same set of 1000 pairs to the participants to find that the number of erroneous responses to be  $234$ (whereas with AND queries it was 105). The error-rate for AND queries is less because it is more likely that the participants are familiar with at least one of the movies, in stead of knowing about both. This substantiates the claim that AND queries are easier to answer for workers. Since this number of errors is too high, we have compared the performance of `same cluster' queries (Algorithm \ref{alg:same_cluster}) with AND queries (Algorithm \ref{alg1}) in a synthetically generated dataset via simulation with two hundred elements. For each value of noise parameter $q=0,0.15,0.30$ and $0.45$, we ran Algorithm \ref{alg:same_cluster} (with different values of $|\mathcal{S}|$ in $\{2,4,6,\dots,60\}$) and Algorithm \ref{alg1} (with different values of $d$ in $\{2,4,6,\dots,60\}$) and computed the average proportion of incorrect labels where the average is computed after repeating each instance for $100$ times. Note that Algorithm \ref{alg:same_cluster} produces a clustering (and not a labeling) but we find the number of incorrect labels by computing the number of errors in the best assignment of labels to the clusters. In addition, if the clustering in the first stage of Algorithm \ref{alg:same_cluster} does not produce two connected components, then we consider all the elements to be incorrectly labeled. The results for `same-cluster' queries (Algorithm \ref{alg:same_cluster}) and `AND' queries (Algorithm \ref{alg1}) are shown in Fig \ref{fig:xor} and Fig \ref{fig:and} respectively. As expected from the theoretical guarantees, AND queries (Algorithm \ref{alg1}) outperform `same cluster' queries used as in Algorithm \ref{alg:same_cluster}. Further, for recovery with `same cluster' queries, we have used the popular spectral clustering algorithm (Figure \ref{fig:comp}) with normalized cuts \cite{ng2002spectral}, for which there are no theoretical guarantees in terms of number of queries. The detailed results obtained are plotted in Figure~\ref{fig:comp2} below. Note that, for `low noise' regime the `same cluster' queries using spectral clustering can outperform our algorithm with AND queries in practice, albeit, as mentioned above, we do not have good theoretical guarantees. 

\begin{figure}[htbp]
  \centering
  
  \begin{minipage}[t]{0.47\textwidth}
 \includegraphics[width=\textwidth]{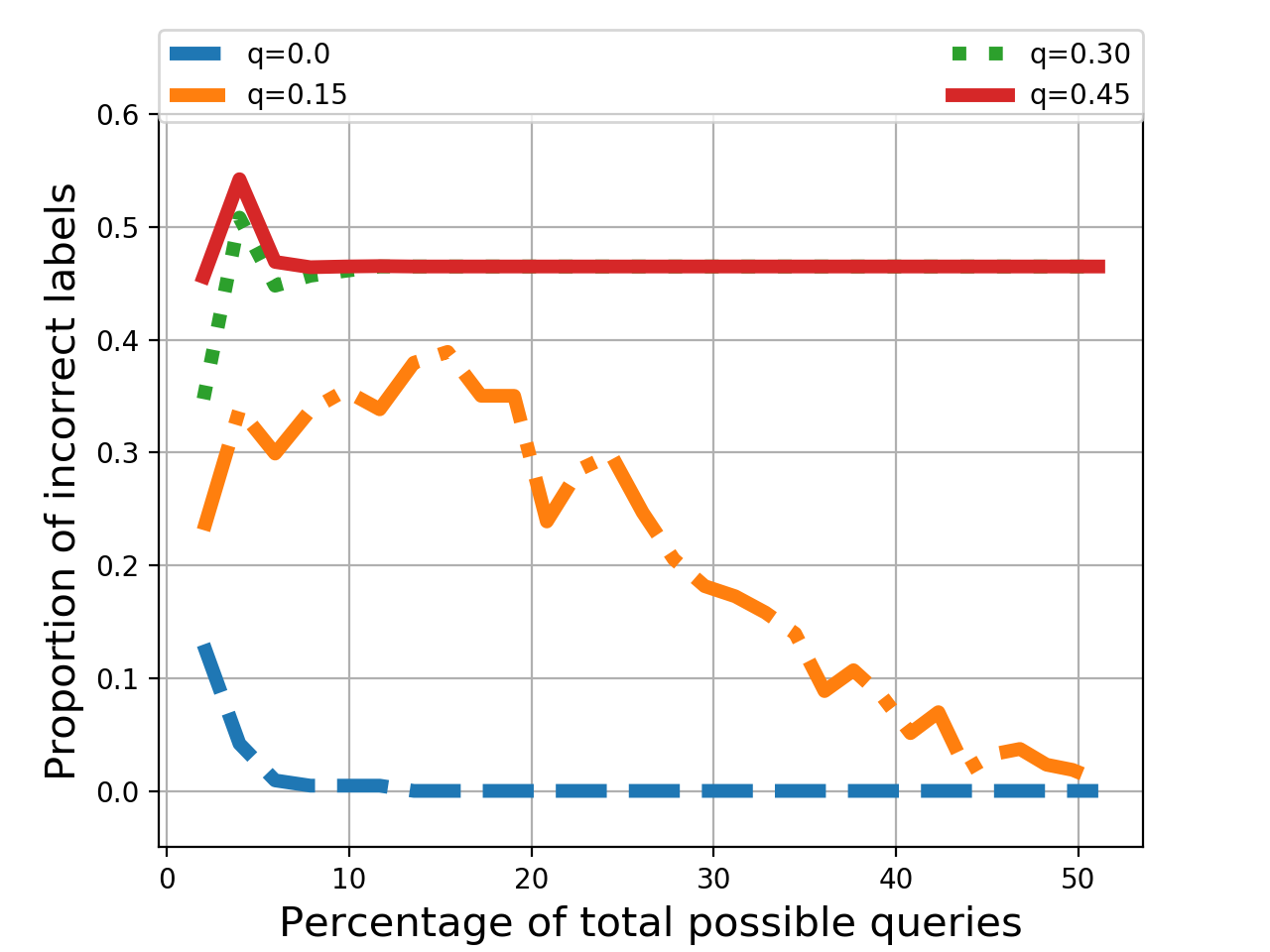}
  \caption{Average proportion of incorrect labels for $200$ elements with $|\mathcal{S}|=2,4,\dots,60$ (each instance repeated for 100 times) plotted using Algorithm \ref{alg:same_cluster}. The $x$-axis label is the percentage of the total possible queries ($19900$) used. \label{fig:xor}}
     \end{minipage}
    \hfill  
    \begin{minipage}[t]{0.47\textwidth}  
     \includegraphics[width=\textwidth]{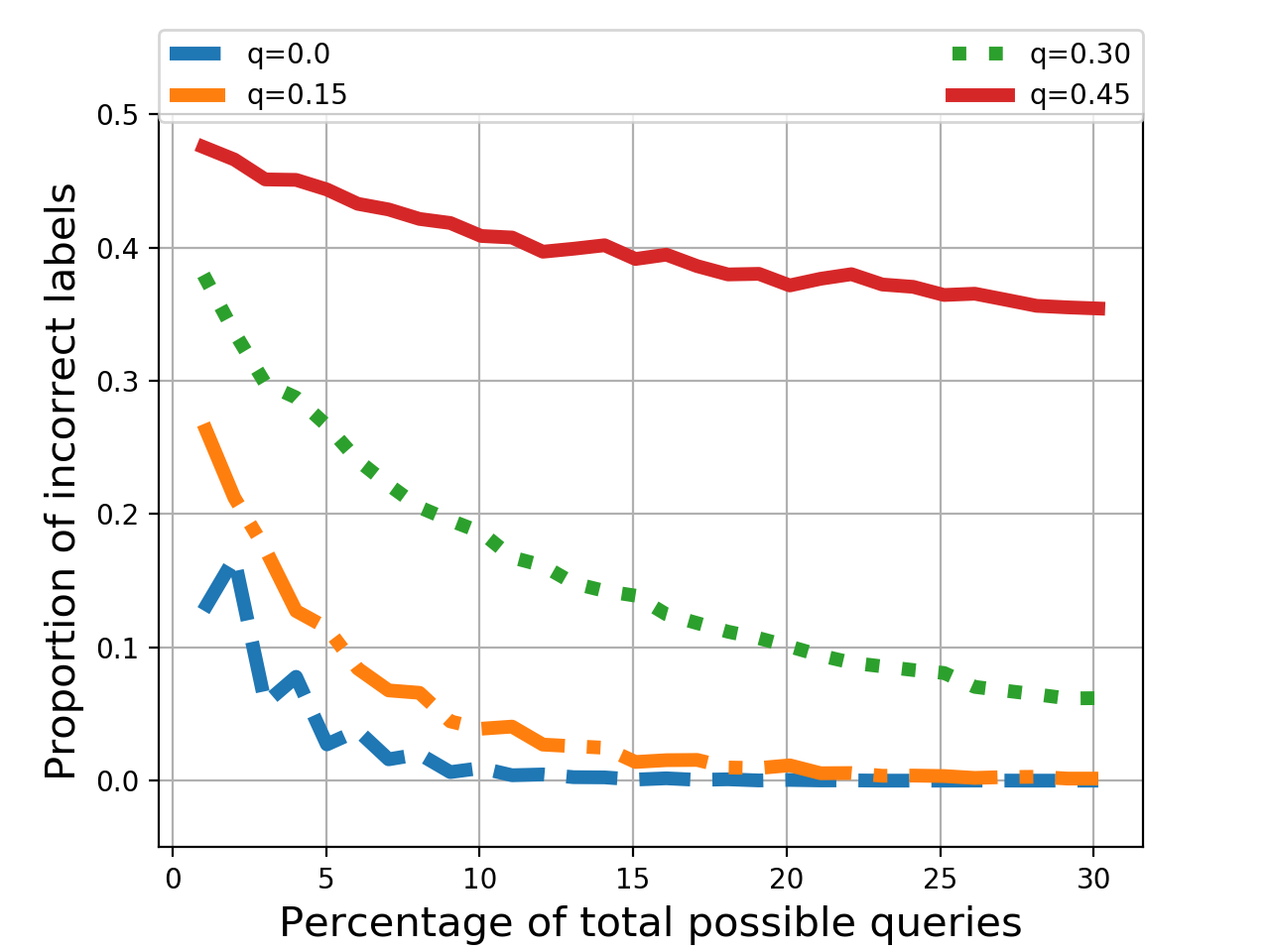}
      \caption{Average proportion of incorrect labels for $200$ elements with $d=2,4,\dots,60$ (each instance repeated for 100 times) plotted using Algorithm \ref{alg1}. The $x$-axis label is the percentage of the total possible queries ($19900$) used. \label{fig:and}}
       \end{minipage}
\end{figure}

\begin{figure}[htbp]
  \centering
  
  \begin{minipage}[t]{0.49\textwidth}
 \includegraphics[width=\textwidth]{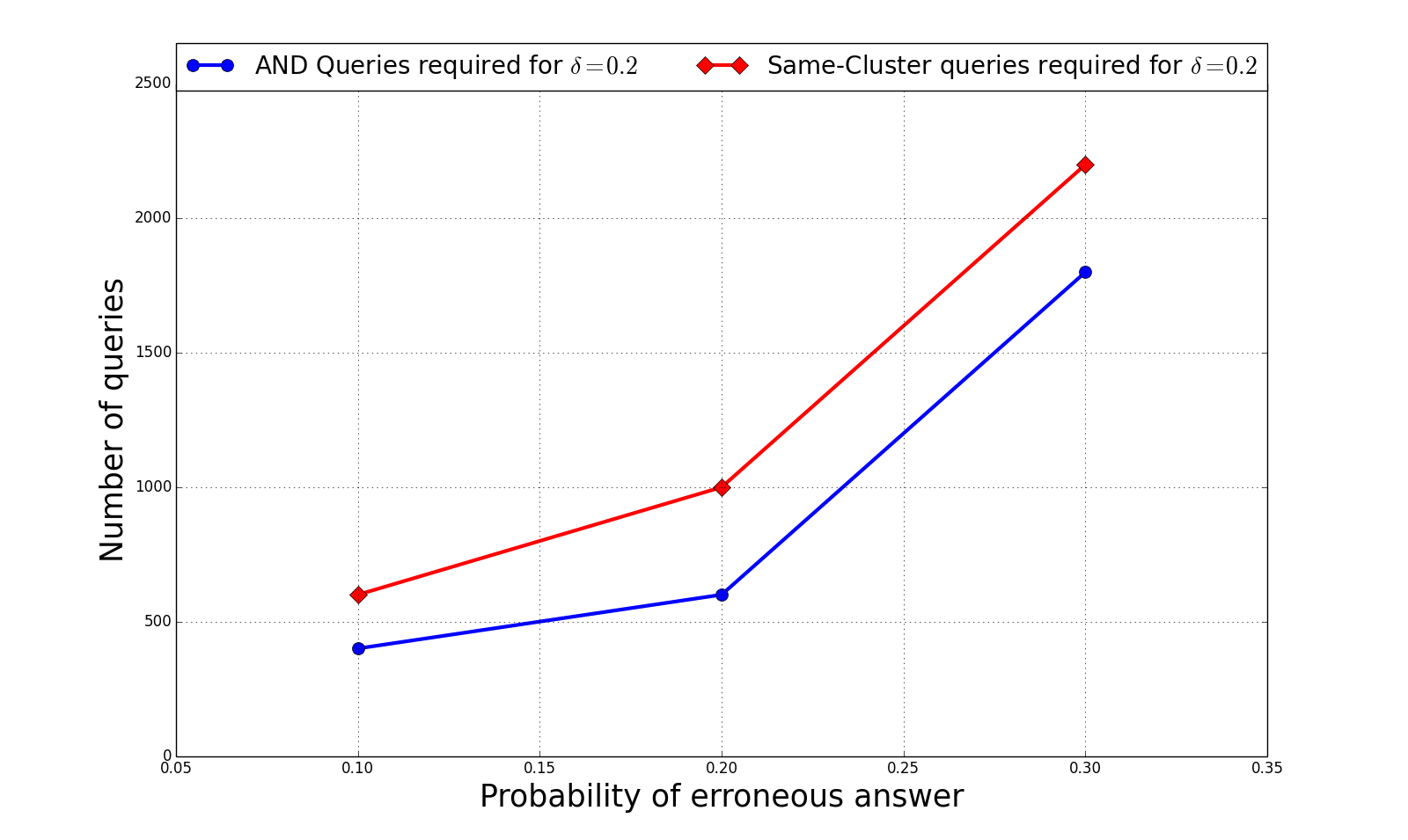}
    \caption{Comparison of `same cluster' query with AND queries when both achieve 80\% accuracy\label{fig:comp}}  \end{minipage}
    \hfill  
    \begin{minipage}[t]{0.49\textwidth}  
     \includegraphics[width=\textwidth]{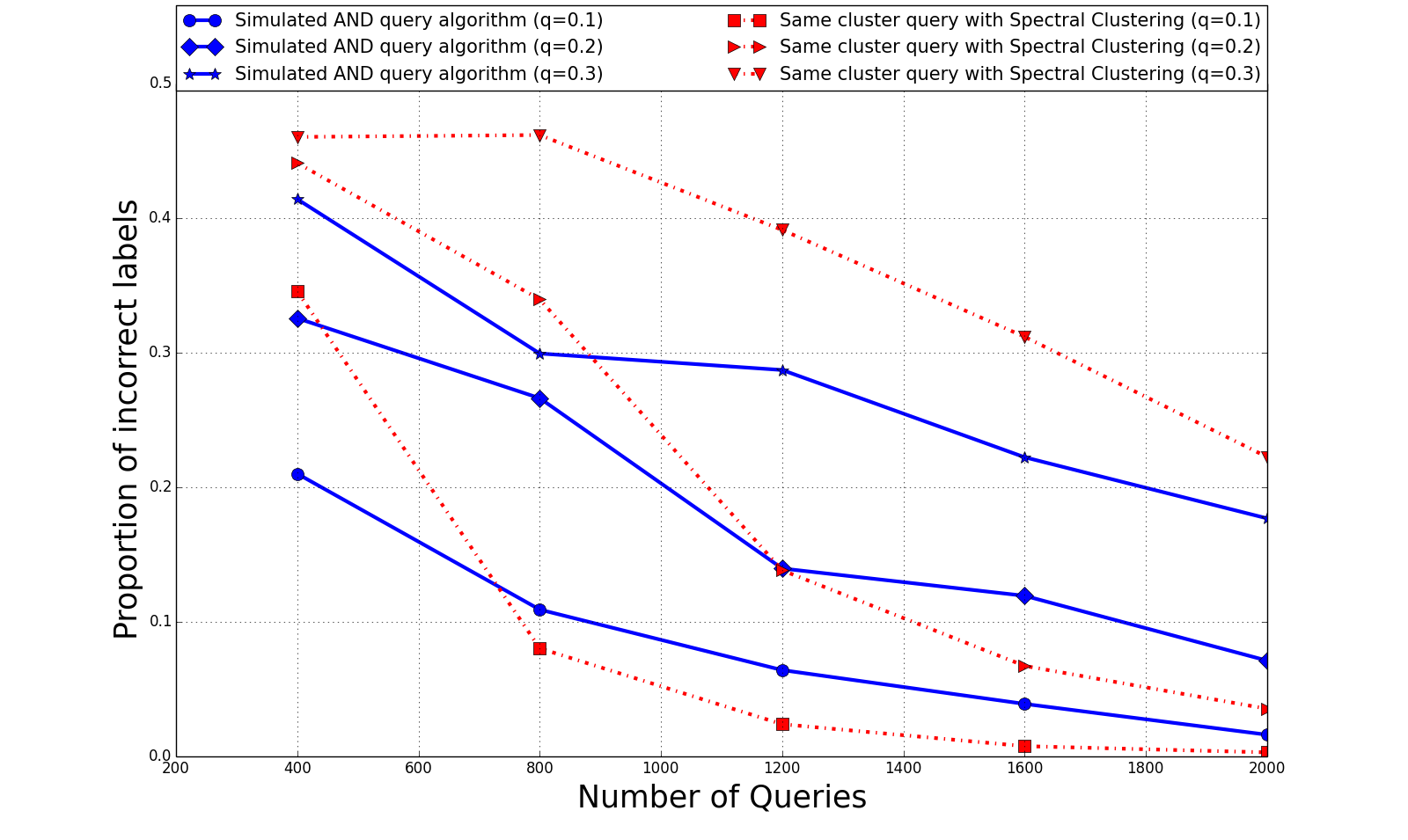}
    \caption{Comparison of performance of `same cluster' query with AND queries on a randomly generated dataset for varying probability of erroneous answers and varying number of queries. The AND querying methods performs well in the high noise regime.\label{fig:comp2}}
    \end{minipage}
\end{figure}

\section{Proof of Theorem \ref{thm:1} and comparison with \cite{miller2001bounds}}\label{sec:thm1}
We have assumed that the label vector $\bfX \in \{0,1\}^n$ consists of i.i.d. Bernoulli random variables with $\Pr(1) =p < \frac12.$
We will define the typical set $\textup{Typ}(p)\equiv \{\mathbf{Z} \in \{0,1\}^{n} \mid np-n^{2/3} \le \textup{wt}(\mathbf{Z}) \le np+n^{2/3}\},$ where $\textup{wt}(\cdot)$ denote the Hamming weight of the argument. We know that (See \cite{cover2012elements}) $\Pr(\mathbf{X} \in \textup{Typ}(p)) \ge 1-\epsilon_n$, with $\epsilon_n \to 0$ as $n \to \infty$. For a fixed vector $\mathbf{X}$ in the typical set, let us  denote $A_{t,\mathbf{X}}$ as the number of vectors in the typical set that are at a distance $t$ from $\mathbf{X}$. Let us denote the weight of the vector $\mathbf{X}$ by $np+s$ where $|s|\le n^{2/3}$. It is easy to see that $A_{t,\mathbf{X}}=\sum_{j: |j| \le 2n^{2/3}}{np+s \choose \frac{t}{2}-j}{n-np-s \choose \frac{t}{2}+j}$ {\color{black}in that case} \footnote{This is for the case when $t$ is even. For odd $t$, we will have $A_{t,\bfX}=\sum_{j}{np+s \choose \frac{t-1}{2}-j}{n-np-s \choose \frac{t+1}{2}+j}+{np+s \choose \frac{t+1}{2}-j}{n-np-s \choose \frac{t-1}{2}+j}$}. Let $A_{t}$ be defined as $\max_{\mathbf{X}\in \textup{Typ}(p)} A_{t,\mathbf{X}}$ which is independent of the vector $\mathbf{X}$ and suppose that the weight of the vector $\mathbf{X}$ which maximizes $A_{t,\mathbf{X}}$ is $np+s^{\ast}$ for some $ s^{\ast}: |s^{\ast}| \le n^{2/3}$. For two binary vectors $\mathbf{X},\mathbf{Y} \in \{0,1\}^{n}$ we will denote $d_{H}(\mathbf{X},\mathbf{Y})$ to be the Hamming distance between the two vectors. Now, we can rewrite the probability of error $P_e$ as below:

\begin{align*}
P_e &= \sum_{\mathbf{X} \in \textup{Typ}(p)} \Pr(\mathbf{X})\Pr_{\mathbf{Q} \sim \mathcal{Q}}(\Psi(\mathbf{Q}\mathbf{X}) \neq \bfX) \\
&+\sum_{\bfX \not \in \textup{Typ}(p)} \Pr(\bfX)\Pr_{\mathbf{Q} \sim \mathcal{Q}}(\Psi(\mathbf{Q}\bfX) \neq \bfX)     \\
 &\le \sum_{\bfX \in \textup{Typ}(p)} \Pr(\bfX)\Pr_{\mathbf{Q} \sim \mathcal{Q}}(\Psi(\mathbf{Q}\bfX) \neq \bfX) \\
& +\sum_{\bfX \not \in \textup{Typ}(p)}\Pr(\bfX)   \\   
&= \sum_{\bfX \in \textup{Typ}(p)} \Pr(\bfX)\Pr_{\mathbf{Q} \sim \mathcal{Q}}(\exists \bfX' \in \textup{Typ}(p) \textup{ such that } \\ &\mathbf{Q}\bfX=\mathbf{Q}\mathbf{X'}) + \epsilon_n   \\   
& \le \epsilon_n+ \sum_{\bfX \in \textup{Typ}(p)} \Pr(\bfX) \times \\
&\sum_{t=1}^{2np+2n^{2/3}} \sum_{\mathbf{X'} \in \textup{Typ}(p) \mid d_{H}(\bfX,\bfX')=t} \Pr_{\mathbf{Q} \sim  \mathcal{Q}}(\mathbf{Q}\bfX=\mathbf{Q}\bfX')  \\
&= \epsilon_n+\sum_{\bfX \in \textup{Typ}(p)} \Pr(\bfX)\times \\
&\sum_{t=1}^{2np+2n^{2/3}} A_{t,\bfX}\Pr_{\mathbf{Q} \sim \mathcal{Q}} (\mathbf{Q}\mathbf{Z}=0 \mid \textup{wt}(\mathbf{Z})=t)  \\
&\le \sum_{t=1}^{2np+2n^{2/3}} A_t\Pr_{\mathbf{Q} \sim \mathcal{Q}} (\mathbf{Q}\mathbf{Z}=0 \mid \textup{wt}(\mathbf{Z})=t) +\epsilon_n \\
&\le \sum_{t \textup{ odd, }t=1}^{2np+2n^{2/3}} A_t\Pr_{\mathbf{Q} \sim \mathcal{Q}} (\mathbf{Q}\mathbf{Z}=0 \mid \textup{wt}(\mathbf{Z})=t) \\
&+\sum_{t \textup{ even, }t=2}^{2np+2n^{2/3}} A_t\Pr_{\mathbf{Q} \sim \mathcal{Q}} (\mathbf{Q}\mathbf{Z}=0 \mid \textup{wt}(\mathbf{Z})=t)+\epsilon_n \\
&\le P_e^{\textup{odd}}+P_e^{\textup{even}}+\epsilon_n,
\end{align*}

where the first two terms in the sum are termed $P_e^{\textup{odd}}$ and $P_e^{\textup{even}}$ respectively.
Now, we will use the following Lemma from \cite{miller2001bounds}
\begin{lemma}
If $tc$ is odd, then 
\begin{align*}
\Pr_{\mathbf{Q} \sim \mathcal{Q}} (\mathbf{Q}\mathbf{Z}=0 \mid \textup{wt}(\mathbf{Z})=t)=0
\end{align*}
otherwise, we will have the following two upper bounds:
\begin{align*}
&\Pr_{\mathbf{Q} \sim \mathcal{Q}} (\mathbf{Q}\mathbf{Z}=0 \mid \textup{wt}(\mathbf{Z})=t) \\
& \le {m \choose \frac{tc}{2}}\Big(\frac{tc}{2m}\Big)^{tc} \quad \textup{ for } t\le \frac{2m}{c}, \\
&\Pr_{\mathbf{Q} \sim \mathcal{Q}} (\mathbf{Q}\mathbf{Z}=0 \mid \textup{wt}(\mathbf{Z})=t) \\
& \le (m\Delta+1)2^{-m}\Big(1+\Big(1-\frac{2t}{n}\Big)^{\Delta}\Big)^{m} .
\end{align*}
\end{lemma} 

Therefore, after substituting the values of $A_t$, we can write:
\begin{align*}
& P_e^{\textup{even}} \le \\
& \sum_{t \textup{ even, }t=2}^{\beta n}{m \choose \frac{tc}{2}}\Big(\frac{tc}{2m}\Big)^{tc} \sum_j {np+s^{\ast} \choose \frac{t}{2}-j}{n-np-s^{\ast} \choose \frac{t}{2}+j}\\
&+ \sum_{t \textup{ even, } t=\beta n}^{2np+2n^{2/3}}\frac{m\Delta+1}{2^{m}}\Big(1+\Big(1-\frac{2t}{n}\Big)^{\Delta}\Big)^{m} \times \\
& \sum_j {np+s^{\ast} \choose \frac{t}{2}+j}{n-np-s^{\ast} \choose \frac{t}{2}-j}
\end{align*}
for some $0 \le \beta \le \frac{2}{\Delta}$ to be determined later. 
Let us define the function $f:\mathbb{Z}\rightarrow \mathbb{R}$ which takes a positive even number as input:
\begin{align*}
f(x)=\sum_{j}{np+s^{\ast} \choose \frac{x}{2}-j}{n-np-s^{\ast} \choose \frac{x}{2}+j}{m \choose \frac{cx}{2}}\Big(\frac{cx}{2m}\Big)^{cx}.
\end{align*}
Notice that \\
\begin{align*}
&\frac{f(x+2)}{f(x)} \\
&=\frac{\sum_{j} {np +s^{\ast}\choose \frac{x}{2}-j+1}{n-np-s^{\ast} \choose \frac{x}{2}+j+1}{m \choose \frac{cx}{2}+c}\Big(\frac{cx+2c}{2m}\Big)^{cx+2c}}{\sum_{j}{np+s^{\ast} \choose \frac{x}{2}-j}{n-np-s^{\ast} \choose \frac{x}{2}+j}{m \choose \frac{cx}{2}}\Big(\frac{cx}{2m}\Big)^{cx}} \\ \nonumber
&\le \frac{{m \choose \frac{cx}{2}+c}\Big(\frac{cx+2c}{2m}\Big)^{cx+2c}}{{m \choose \frac{cx}{2}}\Big(\frac{cx}{2m}\Big)^{cx}} \cdot \max_j \frac{{np+s^{\ast} \choose \frac{x}{2}-j+1}{n-np-s^{\ast} \choose \frac{x}{2}+j+1}}{{np+s^{\ast} \choose \frac{x}{2}-j}{n-np-s^{\ast} \choose \frac{x}{2}+j}} \\ \nonumber
&=\max_j \Big(\frac{np+o(n)}{x/2-j+1}\cdot \frac{n-np-o(n)}{x/2+j+1}\Big) \\
& \times  \frac{\prod_{i=1}^{c} (m-cx/2-i)}{\prod_{i=1}^{c} (cx/2+i)}\cdot \frac{\Big(\frac{cx+2c}{2m}\Big)^{cx+2c}}{\Big(\frac{cx}{2m}\Big)^{cx+2c-2c}} \\ \nonumber
&< \frac{np(n-np)+o(n^2)}{(\frac{x}{2}+1)^{2}}\cdot \Big(\frac{2m}{cx}-1\Big)^{c}  \\
& \cdot\Big(1+\frac{2}{x}\Big)^{cx+2c}\cdot \Big(\frac{cx}{2m}\Big)^{2c} \\ \nonumber
&< \frac{np(n-np)+o(n^2)}{(\frac{x}{2}+1)^{2}} \\
& \cdot \Big(\frac{2m}{cx}-1\Big)^{c} \cdot\Big(1+\frac{2}{x}\Big)^{3cx}\cdot \Big(\frac{cx}{2m}\Big)^{2c} \\ \nonumber
&< \frac{(np(n-np)+o(n^2))e^{3c/2}}{(\frac{x}{2}+1)^{2}} \cdot \Big(\frac{cx}{2m}\Big)^{c} \\ \nonumber
&< \frac{(np(n-np)+o(n^2))e^{3c/2}}{x^2/4} \cdot \Big(\frac{cx}{2m}\Big)^{c} \\ \nonumber
&< \frac{c^2(np(n-np)+o(n^2)e^{3c/2}}{m^2} \cdot \Big(\frac{cx}{2m}\Big)^{c-2}  \\ \nonumber
& < \frac12,
\end{align*}
when $x < \frac{2n}{\Delta}\Big(\frac{1}{2\Delta^2p(1-p)e^{3c/2}}\Big)^{\frac{1}{c-2}}$ (and using the fact $nc = m\Delta$).
Hence, for $\beta=\frac{2}{\Delta}\Big(\frac{1}{2\Delta^2p(1-p)e^{3c/2}}\Big)^{\frac{1}{c-2}}$, the condition $\frac{f(x+2)}{f(x)} <\frac{1}{2}$ is  satisfied for all $x<\beta n$.  In that case we can rewrite $P_e^{\textup{even}}$ as
\begin{align}{\label{eq:new}}
& P_e^{\textup{even}} \le 2f(2)+  \sum_{t \textup{ even, } t=\beta n}^{2np+2n^{2/3}} \frac{m\Delta+1}{2^{m}}\Big(1+\Big(1-\frac{2t}{n}\Big)^{\Delta}\Big)^{m} \\ \nonumber
& \times \sum_j {np+s^{\ast} \choose \frac{t}{2}+j}{n-np-s^{\ast} \choose \frac{t}{2}-j}
\end{align}
where $f(2) \le n^2\Big(p(1-p)+p^{2}+(1-p)^{2}+o(1)\Big){m \choose c}\Big(\frac{c}{m}\Big)^{2c} \le n^{2-c}\Big(p(1-p)+p^{2}+(1-p)^{2}+o(1)\Big)(\Delta c)^{c}$.

Let us denote by $\sigma$ the second term of the right hand side in Eq.~\eqref{eq:new} which we can upper bound as follows:
\begin{align*}
\sigma &=\sum_{t=\beta n:t \textup{ is even}}^{2np+2n^{2/3}} \sum_j {np+s^{\ast} \choose \frac{t}{2}+j}{n-np-s^{\ast} \choose \frac{t}{2}-j} \\
&\times \frac{m\Delta+1}{2^{m}}\Big(1+\Big(1-\frac{2t}{n}\Big)^{\Delta}\Big)^{m} \\
&\le \sum_{t=\beta n:t \textup{ is even}}^{2np+2n^{2/3}} 2^{npH(\frac{t}{2np})+(n-np)H(\frac{t}{2n(1-p)})} \\
&\times \frac{nc+1}{2^{nc/\Delta}}\Big(1+\Big(1-\frac{2t}{n}\Big)^{\Delta}\Big)^{nc/\Delta} \\
\end{align*}
where we have used the fact that $\sum_{j\mid \frac{j}{a}=o(1)} {a \choose \gamma a+j} \le 2^{aH(\gamma)}$. Hence, on taking logarithm on both sides, we will have 
\begin{align*}
\frac{\log \sigma}{n} &\le \max_{\beta \le x \le 2p} pH\Big(\frac{x}{2p}\Big)+(1-p)H\Big(\frac{x}{2(1-p)}\Big) \\
&+\frac{c}{\Delta}\log \Big(1+\Big(1-2x \Big)^{\Delta}\Big)-\frac{c}{\Delta}+\frac{o(n)}{n}
\end{align*}
Therefore, if $\max_{\beta \le x \le 2p} pH\Big(\frac{x}{2p}\Big)+(1-p)H\Big(\frac{x}{2(1-p)}\Big) 
+\frac{c}{\Delta}\log \Big(1+\Big(1-2x \Big)^{\Delta}\Big)<\frac{c}{\Delta}$, then $\sigma$ decreases exponentially with $n$.
In the same way as above, we can also bound $P_{e}^{\textup{odd}}$ as follows:
\begin{align}{\label{eq:new2}}
&P_e^{\textup{odd}} \le 2f(1)+\sum_{t=\beta n:t \textup{odd}}^{2np+2n^{2/3}} \frac{m\Delta+1}{2^{m}}\Big(1+\Big(1-\frac{2t}{n}\Big)^{\Delta}\Big)^{m} \nonumber \\
& \times \sum_j {np+s^{\ast} \choose \frac{t}{2}+j}{n-np-s^{\ast} \choose \frac{t}{2}-j}
\end{align}
Now the first term is small compared to $f(2)$ (i.e. $f(1)=o(f(2))$) and the second term again goes to zero asymptotically with $n$ when the condition $\max_{\beta \le x \le 2p} pH\Big(\frac{x}{2p}\Big)+(1-p)H\Big(\frac{x}{2(1-p)}\Big)+\frac{c}{\Delta}\log \Big(1+\Big(1-2x \Big)^{\Delta}\Big)<\frac{c}{\Delta}$ is satisfied. Substituting $\frac{m}{n}=\frac{c}{\Delta}$ in the above condition, we get back the statement of Theorem \ref{thm:1}.

\noindent Here we state the results of \cite{miller2001bounds} in our notations:
\begin{theorem}[Theorem 3, \cite{miller2001bounds}]\label{thm:miller}
Consider the random ensemble $\mathcal{Q}$ of Query matrices $\mathbf{Q}_{m \times n}$ described above and binary source $\textup{Ber}(p)$ from which $n$ labels are sampled independently. Let $c,\Delta$ be the left degree and the right degree of the ensemble such that $3 \le c < \Delta$ and let $\beta=\frac{2}{\Delta}\exp(-12-\frac{6 \ln \Delta}{c})$. If there exists $0 \le \gamma \le \frac{1}{2}$ such that 
\begin{align*}
\max_{\beta \le x \le \gamma} \frac{x\log \sqrt{2p(1-p)}+H(x)}{1-\log \Big(1+\Big(1-2x \Big)^{\Delta}\Big)}<\frac{m}{n}
\end{align*}
and 
\begin{align*}
\frac{m}{n} > \frac{H(p)}{1-\log \Big(1+\Big(1-2\gamma \Big)^{\Delta}\Big)}
\end{align*}
then the average probability of error $P_e$ goes to zero asymptotically.
\end{theorem}

It is difficult to compare analytically the bound of Thm.~\ref{thm:miller} with our result  (Thm.~\ref{thm:1}). Note that, the number of queries is an increasing function of the prior probability ($p$). We have plotted the number of  queries against the prior probabilities obtained from these two theorems for  $\Delta=7,10$.
It can be seen that in some points, our analysis (blue dots) is tighter than the expression in \cite{miller2001bounds} (red dots), while being same/similar in others.
\begin{figure*}[tb]
  \begin{subfigure}[t]{0.49\textwidth}
     \includegraphics[height=1.9in]{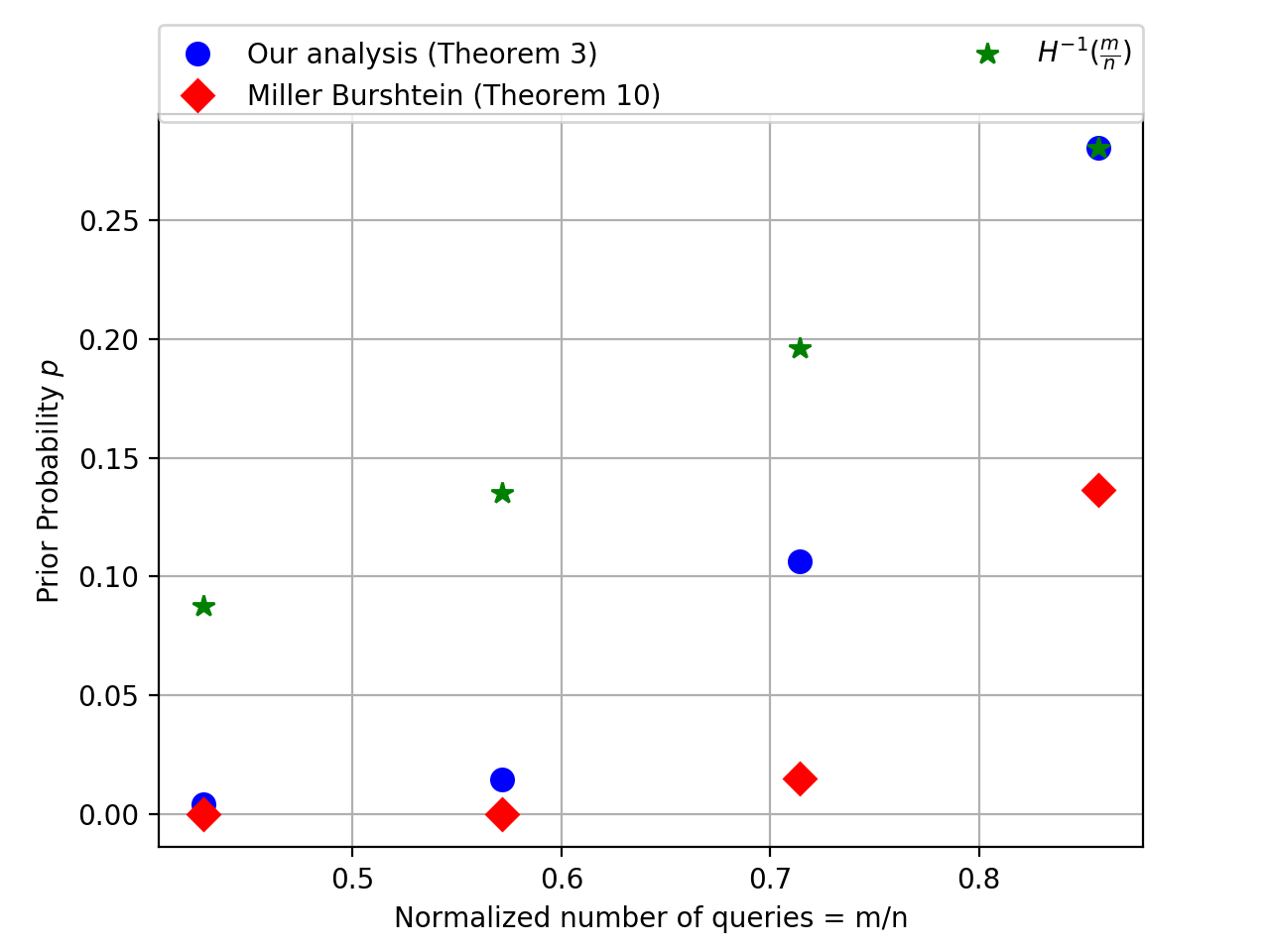}
     \caption{\small Comparison of Theorem \ref{thm:1} and \ref{thm:miller} for $\Delta=7$.}
          ~\label{fig:Delta=7compare.png}
  \end{subfigure}
  \hfill
\begin{subfigure}[t]{0.49 \textwidth}
   \includegraphics[height=1.9in]{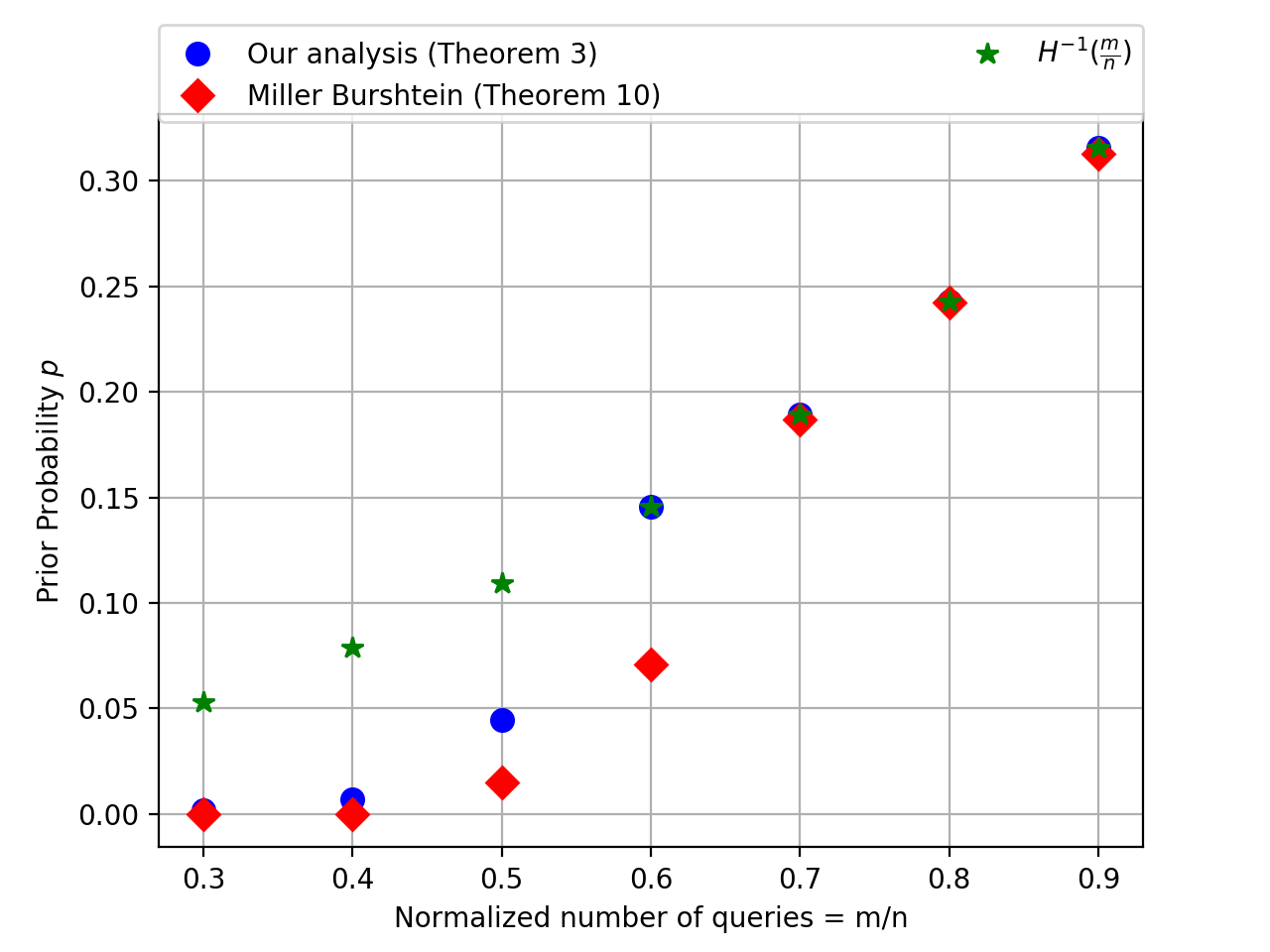}
   \caption{\small Comparison of Theorem \ref{thm:1} and \ref{thm:miller} for $\Delta=10$.}
       ~\label{fig:Delta=10compare.png}
 \end{subfigure}%
\hfill
 \caption{\small Comparison of our analysis and the analysis in \cite{miller2001bounds} for the random ensemble $\mathcal{Q}$}\label{fig:first1}
\end{figure*}

\section{Proof of Theorem \ref{thm:3}}\label{sec:thm3}
We will need the following definitions for this proof inspired by \cite{chandar2010sparse}.
\begin{definition}
If the number of queries is $m$ and the number of input labels is $n$ then we define rate as the relative number of queries or $R=\frac{m}{n}$
\end{definition}
\begin{definition}
The Rate-Distortion function $R(\delta)$ is the infimum of the feasible rates such that the scheme is $(1-\delta)$-good. \end{definition}
\begin{definition}
The Distortion-Rate function $\delta(R)$ is the infimum of all $\delta$, for $(1-\delta)$-good schemes,  when the rate is $R$.
\end{definition}
\begin{definition}
The set of reconstructed label vectors are called {\em codewords}. Since the rate is R, our problem is to define the querying scheme $Q:\{0,1\}^{n} \rightarrow \{0,1\}^{nR}$ and a  recovery $ \{0,1\}^{nR} \rightarrow  \{0,1\}^{n}$.  We have a bijective mapping from query answers $\bfY$ to $\hat{\bfX}$. Hence the total number of possible codewords is $2^{nR}$.
\end{definition}
\begin{proof}[Proof of Theorem \ref{thm:3}]
We are interested in finding a lower bound on the distortion that we will achieve if we use a rate $R(\tilde{\delta})$ for the model where $R(\tilde{\delta})$ is the minimum rate for distortion $\tilde{\delta}$ achieved optimally in the unconstrained case. Now suppose that the distortion achieved in the model at rate $R(\tilde{\delta})$ is $\delta=\tilde{\delta}+\epsilon$ and hence we want a lower bound on $\epsilon$  . Since our input labels are typical sequences and it is mapped to a unique codeword,  the reconstructed sequence must be having a per symbol distortion of less than $\tilde{\delta}+\epsilon$ . We will be counting the number of label vector-codeword pairs $(S,T)$ where $S \in \{0,1\}^{n}$ is a label vector and $T \in \{0,1\}^{n}$ is the corresponding codeword for $S$. Let us allow a small extra distortion of $\gamma$. Now, we have 2 ways to count the number of possible pairs. Firstly, from the perspective of the codewords , the number of possible pairs will be $2^{nR}{\rm Vol}(\tilde{\delta}+\epsilon+\gamma)$ where 
${\rm Vol}(\tilde{\delta}+\epsilon+\gamma)$ is the number of label sequences present in the ball of radius $n(\tilde{\delta}+\epsilon+\gamma)$ from a particular codeword. Since there might be repetitions hence we are overcounting and hence this value is definitely an upper bound on the number of pairs. Again we can try to see from the perspective of the label sequences. Now let us say that we have a label sequence $S$ and a corresponding compressed sequence $C$ and the codeword $T$ when no extra distortion is allowed. We will be trying to find out the number of other different codewords $S$ could have mapped to when this extra distortion is allowed. Let us take ball of $n\gamma$ around $S$ and take another label sequence in that ball and call it $\hat{S}$. Let the codeword it was initially mapped to be $\hat{T}$ Now, 
\begin{align*}
|\hat{T}-S| \le |\hat{T}-\hat{S}|+|\hat{S}-S| \le n(\tilde{\delta}+\epsilon+\gamma)
\end{align*}
 Hence $\hat{T}$ is a possible candidate codeword for $S$ if we allow this extra bit of more distortion $\gamma$. Hence we want a lower bound on the number of possible different codewords that can be candidate codewords for $S$ when this extra distortion is allowed. Now we know that there is a bijective mapping from the compressed sequences to the codewords. Let $x$ denote the fraction of bits in $C$ that we can perturb. Then if,
$nRx \Delta <n\gamma$ or 
$x<\frac{\gamma}{R\Delta}$ , then there exists label vectors mapped to those compressed sequences which will be within a ball of $n\gamma$ from $S$ and all the codewords corresponding to those perturbed compressed sequences must be different because of the bijective mapping. Since the total number of label sequences is $2^{nH(p)}$, the total number of pairs that can be calculated in such a way will be $2^{nH(p)+nRH(\tfrac{\gamma}{R\Delta})}$ which is a lower bound on the actual number of pairs. Since $\frac{1}{n}\log (Vol(\tilde{\delta}+\epsilon+\gamma))=H(p)-R(\tilde{\delta}+\epsilon+\gamma)$ we have 
\begin{align*}
R(\tilde{\delta})-R(\tilde{\delta}+\epsilon+\gamma) \geqslant RH(\frac{\gamma}{R\Delta})
\end{align*}
Using the fact that $R(\delta)=H(p)-H(\delta)$, we have 
$$H(\tilde{\delta}+\epsilon+\gamma)-H(\tilde{\delta}) \ge RH(\frac{\gamma}{R\Delta}).$$ Now since entropy is a concave function of the distribution we must have $H(\tilde{\delta}+\epsilon+\gamma) \ge H(\tilde{\delta})+(\epsilon+\gamma)h'(\tilde{\delta})$ where $h'(x)=\log \frac{ (1-x)}{x}$ is the derivative of the binary entropy function. Plugging it into the formula, we have
$$\epsilon h'(\tilde{\delta})  \ge RH(\frac{\gamma}{R \Delta})-\gamma h'(\tilde{\delta})$$ 
Now we want the value of $\gamma$ in order to get the tightest lower bound. Hence, differentiating w.r.t $\gamma$ and setting it to $0$ in order to maximize it, we will have $\gamma=\frac{R\Delta}{1+e^{\Delta h'(\tilde{\delta})}}$
Using this value of $\gamma$, we plug in the original equation and we get 
$$\epsilon \ge -\frac{R\Delta}{1+e^{\Delta h'(\tilde{\delta})}}+\frac{1}{h'(\tilde{\delta})}RH(\frac{1}{1+e^{\Delta h'(\tilde{\delta})}})$$
Expanding the entropy function we have
\begin{align*}
\epsilon &\ge -\frac{R\Delta}{1+e^{\Delta h'(\tilde{\delta})}}+\frac{R}{h'(\tilde{\delta})(1+e^{\Delta h'(\tilde{\delta})})}\log (1+e^{\Delta h'(\tilde{\delta})}) \\
&+\frac{Re^{\Delta h'(\tilde{\delta})}}{h'(\tilde{\delta})(1+e^{\Delta h'(\tilde{\delta})})}\log (1+\frac{1}{e^{\Delta h'(\tilde{\delta})}})
\end{align*}
 Now if use the facts that $\log (1+e^{\Delta h'(\tilde{\delta})}) \ge \Delta h'(\tilde{\delta})$, $\log (1+\frac{1}{e^{\Delta h'(\tilde{\delta})}}) \ge \frac{1}{e^{\Delta h'(\tilde{\delta})}}$ and $R(\tilde{\delta})=H(p)-H(\tilde{\delta})$, we get that 
 $$\epsilon \ge \frac{H(p)-H(\tilde{\delta})}{h'(\tilde{\delta})(1+e^{\Delta h'(\tilde{\delta})})}.$$
\end{proof}


\ifCLASSOPTIONcaptionsoff
  \newpage
\fi




\begin{thebibliography}{10}

\bibitem{abh:16}
E.~Abbe, A.~S. Bandeira, and G.~Hall.
\newblock Exact recovery in the stochastic block model.
\newblock {\em {IEEE} Trans. Information Theory}, 62(1):471--487, 2016.

\bibitem{alon2004probabilistic}
N.~Alon and J.~H. Spencer.
\newblock {\em The probabilistic method}.
\newblock John Wiley \& Sons, 2004.

\bibitem{ashtiani2016clustering}
H.~Ashtiani, S.~Kushagra, and S.~Ben-David.
\newblock Clustering with same-cluster queries.
\newblock In {\em Advances In Neural Information Processing Systems}, pages
  3216--3224, 2016.

\bibitem{buhrman2002bitvectors}
H.~Buhrman, P.~B. Miltersen, J.~Radhakrishnan, and S.~Venkatesh.
\newblock Are bitvectors optimal?
\newblock {\em SIAM Journal on Computing}, 31(6):1723--1744, 2002.

\bibitem{chandar2010sparse}
V.~B. Chandar.
\newblock {\em Sparse graph codes for compression, sensing, and secrecy}.
\newblock PhD thesis, Massachusetts Institute of Technology, 2010.

\bibitem{cover2012elements}
T.~M. Cover and J.~A. Thomas.
\newblock {\em Elements of information theory}.
\newblock John Wiley \& Sons, 2012.

\bibitem{fss:16}
D.~Firmani, B.~Saha, and D.~Srivastava.
\newblock Online entity resolution using an oracle.
\newblock {\em {PVLDB}}, 9(5):384--395, 2016.

\bibitem{gallager1962low}
R.~Gallager.
\newblock Low-density parity-check codes.
\newblock {\em IRE Transactions on information theory}, 8(1):21--28, 1962.

\bibitem{DBLP:journals/corr/GruenheidNKGK15}
A.~Gruenheid, B.~Nushi, T.~Kraska, W.~Gatterbauer, and D.~Kossmann.
\newblock Fault-tolerant entity resolution with the crowd.
\newblock {\em CoRR}, abs/1512.00537, 2015.

\bibitem{gruenheid2015fault}
A.~Gruenheid, B.~Nushi, T.~Kraska, W.~Gatterbauer, and D.~Kossmann.
\newblock Fault-tolerant entity resolution with the crowd.
\newblock {\em arXiv preprint arXiv:1512.00537}, 2015.

\bibitem{karger2011iterative}
D.~R. Karger, S.~Oh, and D.~Shah.
\newblock Iterative learning for reliable crowdsourcing systems.
\newblock In {\em Advances in neural information processing systems}, pages
  1953--1961, 2011.

\bibitem{karger2014budget}
D.~R. Karger, S.~Oh, and D.~Shah.
\newblock Budget-optimal task allocation for reliable crowdsourcing systems.
\newblock {\em Operations Research}, 62(1):1--24, 2014.

\bibitem{lahouti2016fundamental}
F.~Lahouti and B.~Hassibi.
\newblock Fundamental limits of budget-fidelity trade-off in label
  crowdsourcing.
\newblock In {\em Advances in Neural Information Processing Systems}, pages
  5059--5067, 2016.

\bibitem{liu2012variational}
Q.~Liu, J.~Peng, and A.~T. Ihler.
\newblock Variational inference for crowdsourcing.
\newblock In {\em Advances in neural information processing systems}, pages
  692--700, 2012.

\bibitem{makhdoumi2015locally}
A.~Makhdoumi, S.-L. Huang, M.~M{\'e}dard, and Y.~Polyanskiy.
\newblock On locally decodable source coding.
\newblock In {\em Communications (ICC), 2015 IEEE International Conference on},
  pages 4394--4399. IEEE, 2015.

\bibitem{massey1977joint}
J.~L. Massey.
\newblock Joint source and channel coding.
\newblock Technical report, DTIC Document, 1977.

\bibitem{mazumdar2014update}
A.~Mazumdar, V.~Chandar, and G.~W. Wornell.
\newblock Update-efficiency and local repairability limits for capacity
  approaching codes.
\newblock {\em IEEE Journal on Selected Areas in Communications},
  32(5):976--988, 2014.

\bibitem{mazumdar2015local}
A.~Mazumdar, V.~Chandar, and G.~W. Wornell.
\newblock Local recovery in data compression for general sources.
\newblock In {\em Information Theory (ISIT), 2015 IEEE International Symposium
  on}, pages 2984--2988. IEEE, 2015.

\bibitem{mazumdar2017semisupervised}
A.~Mazumdar and S.~Pal.
\newblock Semisupervised clustering, and-queries and locally encodable source coding
\newblock In {\em Advances in Neural Information Processing Systems}
  pages 6489--6499, 2017.


\bibitem{aaai:17}
A.~{Mazumdar} and B.~{Saha}.
\newblock {A theoretical analysis of first heuristics of crowdsourced entity
  resolution}.
\newblock {\em The Thirty-First AAAI Conference on Artificial Intelligence
  (AAAI-17)}, 2017.

\bibitem{mazumdar2017clustering}
A.~Mazumdar and B.~Saha.
\newblock Clustering with noisy queries.
\newblock In {\em Advances in Neural Information Processing Systems (NIPS) 31},
  2017.

\bibitem{mazumdar2017query}
A.~Mazumdar and B.~Saha.
\newblock Query complexity of clustering with side information.
\newblock In {\em Advances in Neural Information Processing Systems (NIPS) 31},
  2017.

\bibitem{miller2001bounds}
G.~Miller and D.~Burshtein.
\newblock Bounds on the maximum-likelihood decoding error probability of
  low-density parity-check codes.
\newblock {\em IEEE Transactions on Information Theory}, 47(7):2696--2710,
  2001.

\bibitem{montanari2008smooth}
A.~Montanari and E.~Mossel.
\newblock Smooth compression, gallager bound and nonlinear sparse-graph codes.
\newblock In {\em Information Theory, 2008. ISIT 2008. IEEE International
  Symposium on}, pages 2474--2478. IEEE, 2008.

\bibitem{ng2002spectral}
A.~Y. Ng, M.~I. Jordan, and Y.~Weiss.
\newblock On spectral clustering: Analysis and an algorithm.
\newblock In {\em Advances in neural information processing systems}, pages
  849--856, 2002.

\bibitem{pananjady2015compressing}
A.~Pananjady and T.~A. Courtade.
\newblock Compressing sparse sequences under local decodability constraints.
\newblock In {\em Information Theory (ISIT), 2015 IEEE International Symposium
  on}, pages 2979--2983. IEEE, 2015.

\bibitem{pang2019coding}
J.~Pang, H.~Mahdavifar, and S.~Pradhan.
\newblock Coding for crowdsourced classification with XOR queries.
\newblock In {\em 2019 IEEE Information Theory Workshop (ITW)},
  pages 1--5, IEEE 2019.


\bibitem{patrascu2008succincter}
M.~Patrascu.
\newblock Succincter.
\newblock In {\em Foundations of Computer Science, 2008. FOCS'08. IEEE 49th
  Annual IEEE Symposium on}, pages 305--313. IEEE, 2008.

\bibitem{prelec2017solution}
D.~Prelec, H.~S. Seung, and J.~McCoy.
\newblock A solution to the single-question crowd wisdom problem.
\newblock {\em Nature}, 541(7638):532--535, 2017.

\bibitem{vempaty2014reliable}
A.~Vempaty, L.~R. Varshney, and P.~K. Varshney.
\newblock Reliable crowdsourcing for multi-class labeling using coding theory.
\newblock {\em IEEE Journal of Selected Topics in Signal Processing},
  8(4):667--679, 2014.

\bibitem{DBLP:conf/icde/VerroiosG15}
V.~Verroios and H.~Garcia{-}Molina.
\newblock Entity resolution with crowd errors.
\newblock In {\em 31st {IEEE} International Conference on Data Engineering,
  {ICDE} 2015, Seoul, South Korea, April 13-17, 2015}, pages 219--230, 2015.

\bibitem{vesdapunt2014crowdsourcing}
N.~Vesdapunt, K.~Bellare, and N.~Dalvi.
\newblock Crowdsourcing algorithms for entity resolution.
\newblock {\em PVLDB}, 7(12):1071--1082, 2014.

\bibitem{vinayak2016crowdsourced}
R.~K. Vinayak and B.~Hassibi.
\newblock Crowdsourced clustering: Querying edges vs triangles.
\newblock In {\em Advances in Neural Information Processing Systems}, pages
  1316--1324, 2016.

\bibitem{viola2012bit}
E.~Viola.
\newblock Bit-probe lower bounds for succinct data structures.
\newblock {\em SIAM Journal on Computing}, 41(6):1593--1604, 2012.

\bibitem{wang2012crowder}
J.~Wang, T.~Kraska, M.~J. Franklin, and J.~Feng.
\newblock Crowder: Crowdsourcing entity resolution.
\newblock {\em PVLDB}, 5(11):1483--1494, 2012.

\bibitem{zhou2012learning}
D.~Zhou, S.~Basu, Y.~Mao, and J.~C. Platt.
\newblock Learning from the wisdom of crowds by minimax entropy.
\newblock In {\em Advances in Neural Information Processing Systems}, pages
  2195--2203, 2012.

\end{thebibliography}
%

%


\begin{IEEEbiographynophoto}{Arya Mazumdar} (S'05-M'13-SM'16)
 is an associate professor in the College of Information and Computer Sciences at the University of Massachusetts Amherst. Prior to this, he was an assistant professor at University of Minnesota-Twin Cities (2013-2015), and  a postdoctoral scholar at Massachusetts Institute of Technology (2011-2012).  Arya obtained the Ph.D.~degree  from  University of Maryland, College Park, where his thesis won a Distinguished Dissertation Award (2011). Arya is a recipient of multiple other awards, including the NSF CAREER award (2015), an EURASIP JSAP Best Paper Award (2020),   and  the  IEEE ISIT Jack K.~Wolf Student Paper Award (2010).
  He is currently serving as an Associate Editor for the IEEE Transactions on Information Theory and as an Area editor for Now Publishers Foundation and Trends in Communication and Information Theory series. Arya's research interests include coding theory (error-correcting codes and related combinatorics), information theory and foundations of machine learning.
 \end{IEEEbiographynophoto}


\begin{IEEEbiographynophoto}{Soumyabrata Pal}
is a PhD student in College of Information and Computer Sciences at the University of Massachusetts Amherst, advised by Professor Arya Mazumdar. He is interested in theoretical machine learning, applied statistics and information theory. Before coming to Amherst, he obtained his B.Tech from Indian Institute of Technology Kharagpur in India in 2016.
\end{IEEEbiographynophoto}




\end{document}